\documentclass[letterpaper,11pt]{article}
\usepackage[in]{fullpage} 
\usepackage[utf8]{inputenc}
\usepackage{hyperref}
\usepackage{amsmath}
\usepackage{amsthm}
\usepackage{amssymb}
\usepackage{mathtools}
\usepackage{xcolor}
\usepackage{bbm}
\usepackage{enumitem, linegoal}
\usepackage{appendix}
\usepackage{cleveref}

\usepackage[most]{tcolorbox}

\newtheorem{theorem}{Theorem}[section]
\newtheorem*{theorem*}{Theorem}
\newtheorem*{definition*}{Definition}
\newtheorem{definition}[theorem]{Definition}
\newtheorem{corollary}[theorem]{Corollary}
\newtheorem{lemma}[theorem]{Lemma}
\newtheorem*{lemma*}{Lemma}
\newtheorem*{discussion*}{Discussion}

\newtheorem{question}[theorem]{Question}

\newtheorem{claim}[theorem]{Claim}

\newtheorem{remark}{Remark}

\renewcommand{\H}{\mathcal{H}}
\newcommand{\X}{\mathcal{X}}

\newcommand{\A}{\mathcal{A}}
\newcommand{\D}{\mathcal{D}}
\newcommand{\I}{\mathcal{I}}
\newcommand{\C}{\mathcal{C}}
\renewcommand{\P}{\mathcal{P}}
\newcommand{\val}{\mathtt{val}}
\newcommand{\col}{\mathtt{col}}
 
\newcommand{\CD}{\mathtt{CD}}

\renewcommand{\L}{\mathtt{L}}
\renewcommand{\S}{\mathtt{S}}
\newcommand{\FCD}{\mathtt{CD}^\star}
\newcommand{\LD}{\mathtt{LD}}

\newcommand{\RepDim}{\mathtt{RepDim}}
\newcommand{\Nat}{\mathbb{N}}
\newcommand{\R}{\mathbb{R}}

\newcommand{\Ind}{\mathbbm{1}}
\newcommand{\expect}{\mathop{\mathbb{E}}}
\newcommand{\poly}{\mathop{\mathtt{poly}}}

\newcommand\restr[2]{{
  \left.\kern-\nulldelimiterspace 
  #1 
  \vphantom{|} 
  \right|_{#2} 
  }}

\title{A Unified Characterization of Private Learnability via Graph Theory}

\author{Noga Alon
\and Shay Moran \and Hilla Schefler \and  Amir Yehudayoff}

\date{}

\begin{document}

\maketitle

\begin{abstract}

We provide a unified framework for characterizing pure and approximate differentially private (DP) learnability. The framework uses the language of graph theory: for a concept class $\mathcal{H}$, we define the contradiction graph $G$ of $\mathcal{H}$. Its vertices are realizable datasets, and two datasets~$S,S'$ are connected by an edge if they contradict each other (i.e., there is a point $x$ that is labeled differently in $S$ and $S'$). Our main finding is that the combinatorial structure of $G$ is deeply related to learning $\mathcal{H}$ under DP. Learning $\mathcal{H}$ under pure DP is captured by the fractional clique number of $G$. Learning $\mathcal{H}$ under approximate DP is captured by the clique number of $G$. Consequently, we identify graph-theoretic dimensions that characterize DP learnability: the \emph{clique dimension} and \emph{fractional clique dimension}. Along the way, we reveal properties of the contradiction graph which may be of independent interest. We also suggest several open questions and directions for future research.

\end{abstract}

\section{Introduction}

Modern machine learning applications often involve handling sensitive data. Differential privacy (DP) \cite{DMNS06} has emerged as a sound theoretical approach to reason about privacy in a precise and quantifiable fashion and has become the gold standard of statistical data privacy~\cite{DR14}. It has also been implemented in practice, notably by Google~\cite{Erlingsson14google}, Apple \cite{apple,apple1}, and in the 2020 US census~\cite{Census20}.
These developments raise the question: 
\begin{center}
	\emph{Which learning tasks can be performed subject to differential privacy?}    
\end{center}
Extensive research has been carried out on this question within the framework of the classical Probably Approximately Correct (PAC) model~\cite{vapnik:68,valiant:84}, 
leading to the development of various characterizations of private learnability.
Beimel, Nissim, and Stemmer introduced a quantity called the representation dimension that characterizes pure DP learnability~\cite{beimel2013characterizing,Beimel19Pure}.
In a follow-up work, Feldman and Xiao found an interesting connection with communication complexity by associating every concept class $\H$ with 
a communication task whose complexity characterizes whether $\H$ is pure DP learnable~\cite{FeldmanX15}. 

Extensive research has also been devoted to studying the question of which learning tasks can be performed subject
to approximate differential privacy, which is comparatively less demanding than pure differential privacy.
Several characterizations of PAC learnability under this less restrictive version have been proven, including finite Littlestone dimension and online learnability~\cite{ALMM19, BunLM20, AlonBLMM22}, 
replicability and reproducibility~\cite{ImpagliazzoLPS22,bun2023stability},
low information complexity, PAC Bayes stability, and other variants of algorithmic stability~\cite{LivniM20,PradeepNG22}.
For a more detailed discussion, please refer to~\cite{MalliarisM22}.

\paragraph{Our Contribution.}
While the definitions of pure and approximate DP are closely related, the characterizations of learning under these distinct privacy constraints differ significantly.
In this work, we devise a unified approach for characterizing both pure and approximate DP learnability.
Our framework is based on graph theory; in particular, it demonstrates a tight link between private learnability and cliques in certain graphs, which we call \emph{contradiction graphs}.

A clique in a graph $G$ is a set $\delta$ of vertices such that every pair of distinct vertices in~$\delta$ 
is connected by an edge. A fractional clique is a standard LP relaxation of a clique.
A function $\delta :V\to [0,1]$ is a fractional clique if for every independent set $I\subseteq V$,
\[\sum_{v\in I}\delta(v) \leq 1.\]
The size of a fractional clique $\delta$ is the sum $\sum_{v\in V}\delta(v)$. Notice that if $\delta(v)\in\{0,1\}$ for all $v$ then $\delta$ is the indicator function of a clique,
and its size is the number of vertices in the clique. 
The clique number of $G$, denoted $\omega(G)$, is the largest size of a clique in $G$. Similarly, the fractional clique number of $G$, denoted $\omega^\star(G)$, is the largest size of a fractional clique in~$G$. Notice that~$\omega^\star(G)\geq \omega(G)$.

\begin{tcolorbox}[colback=white,arc=0mm,boxrule=0.5pt]
\begin{definition*}[Contradiction Graph]
     Let $\H\subseteq\{0,1\}^\X$ be a concept class and let $m\in\mathbb{N}$. The \emph{contradiction graph of order $m$ of~$\H$} is an undirected graph $G_m=G_m(\H)$ whose vertices are datasets of size $m$ that are consistent with $\H$. Two datasets are connected by an edge whenever they contradict each other.
\end{definition*}
\end{tcolorbox}
 In other words, the vertices of the contradiction graph $G_m(\H)$ are $\H$-realizable sequences of length $m$, and $\{S',S''\}$ is an edge if there is $x\in \X$ such that $(x,0)\in S'$ and $(x,1)\in S''$.
Let $\omega_m = \omega(G_m)$ and $\omega^\star_m= \omega^\star(G_m)$ denote the clique and fractional clique numbers of $G_m$.
\begin{enumerate}
	\item We prove that both $\omega_m$ and $\omega^\star_m$ satisfy a polynomial-exponential dichotomy:
	      \begin{itemize}
	      	\item[(i)] For every $\H$, either $\omega_m^\star= 2^m$ for all $m$, or $\omega_m^\star\leq \mathtt{poly}(m)$.
	      	\item[(ii)] For every $\H$, either $\omega_m= 2^m$ for all $m$, or $\omega_m\leq \mathtt{poly}(m)$.
	      \end{itemize}
	\item These dichotomies characterize pure and approximate DP learnability:
	      \begin{itemize}
	      	\item[(i)] $\H$ is pure DP learnable if and only if $\omega_m^\star\leq \mathtt{poly}(m)$.
	      	\item[(ii)] $\H$ is approximately DP learnable if and only if $\omega_m\leq \mathtt{poly}(m)$.
	      \end{itemize}
	      These characterizations yield graph-theoretic dimensions of $\H$ that characterize private learning. Define the clique dimension of $\H$, denoted by $\CD(\H)$, as the largest $m\in\mathbb{N}\cup\{\infty\}$ for which $\omega_m = 2^m$. Analogously, define the fractional clique dimension of $\H$, denoted by $\FCD(\H)$, as the largest $m$ for which $\omega_m^\star = 2^m$. 
	      Thus,
	      \begin{itemize}
            \item[(i)] $\CD^\star(\H)<\infty$ if and only if $\H$ is pure DP learnable.
	      	\item[(ii)] $\CD(\H)<\infty$ if and only if $\H$ is approximately DP learnable.
	      \end{itemize}
\end{enumerate}
Technically, our proofs rely on the fact that the fractional clique and chromatic numbers are equal.
For finite graphs, this fact follows from LP duality, however, in our setting the contradiction graph can be infinite.
Lastly, we prove that the contradiction graph exhibits strong duality: the fractional clique and chromatic numbers are equal for every (possibly infinite) contradiction graph.
This part is based on tools from functional analysis and topology.

Note that in this work we focus on providing a unified and concise framework to characterize and study private learning. 
Important topics, such as regarding the informational or computational complexity are not addressed in this work. 
Nevertheless, given that graph theory is a well-studied area with sophisticated tools and techniques, 
we hope that the equivalence established here will facilitate a deeper integration between learning theory and graph theory. 
Specifically, we anticipate that it will provide new tools and insights to tackle other fundamental questions in learning theory.

\paragraph{Organization.} 
In \Cref{sec:mainres}, we present the main results in greater detail and provide an overview of some of the main proof ideas.
Section~\ref{sec:definitions} contains background and relevant definitions in Learning Theory, Graph Theory, and Differential Privacy.
Sections~\ref{sec:DPvsCD} and~\ref{sec:PDPvsFCD} contain the full proofs, and some of the proofs are delegated to the appendix. 
Finally, in Section~\ref{sec:future_work}, we provide suggestions for future work.


\section{Main Results}\label{sec:mainres}
We use standard definitions and terminology from graph theory, learning theory, and differential privacy; see Section~\ref{sec:definitions} for detailed definitions.

\subsection{Dichotomies and Dimensions}
In this section we present Theorems~\ref{thm:SSP-CD} and~\ref{thm:SSP-FCD} which concern cliques and fractional cliques in the contradiction graph. These results are key in our characterizations of private learnability, but they might also be of independent interest as combinatorial results.

We begin with a basic lemma which shows that the contradiction graph does not contain cliques or fractional cliques of size larger than $2^m$.
\begin{lemma}\label{l:CN_FCN_bounded}
	Let $\H$ be a class and $m\in\Nat$ and let $\omega_m$ and $\omega_m^\star$ denote the clique and fractional clique numbers of $G_m(\H)$.  
	Then,
	$\omega_m\leq\omega^\star_m\leq 2^m$.
\end{lemma}
A short proof of this lemma is provided in \Cref{sec:fundemental}.

\begin{theorem}[Clique Number]\label{thm:SSP-CD}
	Let $\H$ be a class and let $\omega_m$ denote the clique number of $G_m(\H)$.
	Then, exactly one of the following statements holds:
	\begin{enumerate}
		\item $\omega_m = 2^m$ for all $m$.\label{item:ssp-cd1}
		\item  $\omega_m\leq P(m)$ for all $m$, where $P(m)$ is a polynomial.\label{item:ssp-cd2}
	\end{enumerate}
\end{theorem}

\begin{theorem}[Fractional Clique Number]\label{thm:SSP-FCD}
	Let $\H$ be a class and let $\omega^\star_m$ denote the fractional clique number of $G_m(\H)$. Then, exactly one of the following statements holds:
	\begin{enumerate}
		\item $\omega^\star_m = 2^m$ for all $m$.\label{item:ssp-fcd1}
		\item  $\omega^\star_m\leq P(m)$ for all $m$, where $P(m)$ is a polynomial.\label{item:ssp-fcd2}
	\end{enumerate}
\end{theorem}

Theorems~\ref{thm:SSP-CD} and~\ref{thm:SSP-FCD} motivate the following definitions.

\begin{definition}[Clique Dimension]\label{def:CD}
	The clique dimension of a concept class $\H$, denoted~$\CD(\H)$, is defined as follows:
	\[\CD(\H):= \sup\{m : \omega_m=2^m\}\in \Nat\cup\{\infty\},\]
	where $\omega_m$ is the clique number of the contradiction graph $G_m(\H)$.
\end{definition}

\begin{definition}[Fractional Clique Dimension]\label{def:FCD}
	The fractional clique dimension of a concept class~$\H$, denoted $\FCD(\H)$, is defined as follows:
	\[\FCD(\H):= \sup\{m : \omega_m^\star=2^m\}\in \Nat\cup \{\infty\},\]
	where $\omega_m^\star$ is the fractional clique number of the contradiction graph $G_m(\H)$.
\end{definition}
Thus, Theorems~\ref{thm:SSP-CD} and~\ref{thm:SSP-FCD} demonstrate that the clique and fractional clique dimensions satisfy a dichotomy similar to the Sauer-Shelah-Perles (SSP) dichotomy of the VC dimension~\cite{sauer1972density}. 
Theorems~\ref{thm:SSP-CD} and~\ref{thm:SSP-FCD} are key in our characterizations of private PAC learnability. 
This is analogous to the crucial role played by the SSP lemma in the characterization of PAC learnability.

It is worthwhile to note that the polynomial/exponential dichotomy in Theorem~\ref{thm:SSP-FCD} 
is weaker than the one in Theorem~\ref{thm:SSP-CD}.
Specifically, in Theorem~\ref{thm:SSP-CD} the degree of the polynomial $P(m)$ is the clique dimension $\CD(\H)$ (see Lemma~\ref{l:clique_number_bounded}). 
In contrast, our proof of Theorem~\ref{thm:SSP-FCD} does not imply a bound on the degree of $P(m)$ in terms of the fractional clique dimension. Rather, the implied bound depends on the difference $2^{m} - \omega^\star_{m}>0$,
where $m$ is any integer for which this difference is positive ($m$ exists by Lemma~\ref{l:CN_FCN_bounded}).

We leave as an open question to determine whether the fractional clique number $\omega^\star_m$ is upper bounded by a polynomial $P(m)$ whose degree depends \underline{only} on the fractional clique dimension.

\subsubsection{Theorems~\ref{thm:SSP-CD} and~\ref{thm:SSP-FCD}: Technical Overview}

We begin with overviewing the proof of Theorem~\ref{thm:SSP-CD}.

Assuming there exists a natural number $d$ such that $\omega_{d}<2^{d}$, we need to show that $\lvert\delta\rvert\leq\mathtt{poly}(m)$ for every $m$ and for every clique $\delta$ in $G_m$.
The crux of the proof is to show that there exists a \emph{balanced} instance $x_1$ in the following sense: 
\begin{equation}\label{eq:1}
	\lvert\{S\in \delta : (x_1,0)\in S\}\rvert\geq \Omega\Bigl(\frac{\lvert \delta\rvert}{m}\Bigr) \text{ and } \lvert\{S\in \delta : (x_1,1)\in S\}\rvert\geq \Omega\Bigl(\frac{\lvert \delta\rvert}{m}\Bigr).
\end{equation}
To see how \Cref{eq:1} completes the proof, let $\H_{x\to y}=\{h\in\H: h(x)=y\}$.
Notice that at least one of $G_{d-1}(\H_{x_1\to 0})$, $G_{d-1}(\H_{x_1\to 1})$ does not contain a clique of size $2^{d-1}$.
Indeed, if $\delta_0,\delta_1$ are cliques of size $2^{d-1}$ in $G_{d-1}(\H_{x_1\to 0})$ and $G_{d-1}(\H_{x_1\to 1})$, then they can be combined to form a clique of size $2^{d}$ in $G_{d}(\H)$ by adding the example $(x_1,0)$ to every dataset in $\delta_0$ and the example $(x_1,1)$ to every dataset in $\delta_1$, and taking their union. 

Repeating this argument $d-1$ times, we obtain a class $\H'=\H_{x_1\to y_1, \ldots x_{d-1}\to y_{d-1}}$ such that (i) $G_1(\H')$ does not contain a clique of size $2$ and (ii) there is a clique in $G_m(\H')$ whose size is at least ${\lvert \delta\rvert}/{(c\cdot m)^{d-1}}$ for some constant $c$. On the other hand, the first item means that $\lvert \H'\rvert = 1$ and hence every clique in $G_m(\H')$ has size $1$. Thus, $1\geq {\lvert \delta\rvert}/{(c\cdot m)^{d-1}}$, which implies that $\lvert \delta\rvert \leq (c\cdot m)^{d-1}=\mathtt{poly}(m)$ as required.

We prove that there exists a balanced instance (i.e.\ that satisfies \Cref{eq:1}) 
constructively using a greedy procedure (see Lemma~\ref{l:balanced_example} for the short proof).

\medskip

The proof of Theorem~\ref{thm:SSP-FCD} is more involved and it integrates different ideas and techniques such as LP duality, probabilistic arguments, and regret analysis from online learning.

Our objective is to demonstrate that if there exists a natural number $d$ such that $\omega^\star_{d}<2^{d}$, then $\lvert\delta\rvert\leq\mathtt{poly}(m)$ for \emph{every} fractional clique $\delta$ in $G_m$.
The first step in the proof is to apply LP duality, which reduces the latter to showing that \emph{there exists} a fractional coloring of $G_m(\H)$ that employs $\mathtt{poly}(m)$ colors. It turns out that fractional colorings in the contradiction graph have a natural learning theoretic interpretation; they correspond to distributions over hypotheses. 
This correspondence between fractional colorings and distributions over hypotheses implies that it suffices to prove the following statement: there exists a distribution $\mu$ over hypotheses such that for every dataset~$S$ of size $m$ that is realizable by $\H$:
\begin{equation}\label{eq:2}
	\Pr_{h\sim \mu}[h \text{ is consistent with } S]\geq 1/\mathtt{poly}(m).
\end{equation}
We obtain the distribution $\mu$ as follows. 
By the above correspondence, there is a distribution $\mu'$ such that 
$\Pr_{h\sim \mu'}[h \text{ is consistent with } S]\geq 2^{-d} + \epsilon$,
for every realizable dataset~$S$ of size $d$. 
The distribution $\mu$ is obtained by independently sampling $\Theta(\log m /\epsilon^2)$ hypotheses from $\mu'$ and taking their majority vote. 
Interestingly, the analysis showing that $\mu$ satisfies \Cref{eq:2} follows by a reduction to online prediction using expert advice.

\subsection{Private Learnability: Characterizations}
We next present the characterizations of pure and approximate private learnability.

\begin{theorem}[Pure DP Learnability] \label{thm:PDP-FCD}
	The following statements are equivalent for a concept class~$\H$.
	\begin{enumerate}
		\item $\H$ is pure differentially private PAC learnable.\label{itm:PDP}
		\item $\H$ has finite fractional clique dimension.\label{itm:FCD}
	\end{enumerate}
\end{theorem}

\begin{theorem}[Approximate DP Learnability]\label{thm:CD-DP}
	The following statements are equivalent for a concept class~$\H$.
	\begin{enumerate}
		\item $\H$ is approximately differentially private PAC learnable.
		\item $\H$ has finite clique dimension.
	\end{enumerate}
\end{theorem}

Theorems~\ref{thm:PDP-FCD} and~\ref{thm:CD-DP} provide a unified characterization of private PAC learnability 
in terms of cliques and fractional cliques. 
Roughly speaking, these theorems assert that large (fractional) cliques in the contradiction graph 
correspond to tasks that are hard to learn privately.
However, our proofs of Theorems~\ref{thm:PDP-FCD} and~\ref{thm:CD-DP} do not explicitly illustrate this correspondence. 
Instead, our proofs follow an indirect path by linking the clique and fractional clique dimensions
to the representation and Littlestone dimensions, respectively. This implies the stated equivalences since the representation and
Littlestone dimensions characterize pure and approximate DP learnability~\cite{Beimel19Pure,AlonBLMM22}.

It would be interesting to find direct proofs that illustrate the correspondence between large cliques in the contradiction graphs and hard learning tasks. 
A clique of size $k$ in the contradiction graph is simply a set $\mathcal{S}$ of $k$ realizable datasets such that every two distinct datasets $S',S''\in\mathcal{S}$ disagree on some example. Analogously, a fractional clique of size $k$ is a distribution $\nu$ over realizable data sets such that $\Pr_{S\sim \nu}[h \text{ is consistent with } S]\leq 1/k$ for every hypothesis $h$.
Hence, it is quite natural to speculate that there exists a natural and direct conversion between cliques and fractional cliques, and realizable distributions on datasets that are hard for private learning.

\paragraph{Duality, Representation Dimension, and Communication Complexity.}
The fractional clique number is defined via a linear program and as such it has a dual definition.
The latter is called the fractional chromatic number and is defined as follows.
Let $G$ be an undirected graph and let $\I$ denote the family of independent sets in $G$.
A fractional coloring is an assignment $c:\I\to \mathbb{R}_{\geq 0}$ such that for every vertex $v$ we have $\sum_{I:v\in I}c(I) \geq 1 $.
The number of colors in $c$ is defined by $\col(c)=\sum_{I\in \I}c(I)$.
The fractional chromatic number, denoted by $\chi^\star(G)$, is the infimum number of colors in a fractional coloring.
Notice that $\chi^\star(G)\leq \chi(G)$, where $\chi(G)$ is the chromatic number of $G$. This holds because any coloring defines a fractional coloring by assigning~$1$ to each color class.

Fractional colorings of the contradiction graph have a natural learning theoretic interpretation.
A fractional coloring $c$ of $G_m(\H)$ corresponds to a distribution $\mu$ over hypotheses such that for every realizable dataset $S$ of size $m$:
\begin{equation}\label{eq:cp}
	\Pr_{h\sim \mu}[h\text{ is consistent with }S] \geq \frac{1}{\col(c)}.    
\end{equation}

The fractional chromatic number of the contradiction graph is tightly linked to the representation dimension (see \Cref{def:repdim}).
The latter is a dimension introduced by~\cite{Beimel19Pure} to characterize pure DP learnability.
Roughly, the representation dimension of $\H$ is the minimal integer $d$ for which there exists a distribution
$\P$ over hypothesis \underline{classes} of size $\leq 2^d$ such that for every realizable distribution~$\D$,
a random class $H\sim \P$ contains with probability at least $3/4$ an hypothesis $h=h_\D$ whose loss with respect to $\D$ is at most $1/4$. 

The representation dimension can be interpreted as a lossy variant of the fractional chromatic number.
Indeed, consider a uniformly sampled hypothesis $h$ drawn from a random class $H\sim \P$, 
and pick $\D$ to be the uniform distribution over the examples in a realizable dataset $S$. 
Thus, with probability at least $\frac{3}{4}2^{-d}$, the random hypothesis $h$ 
classifies correctly $3/4$ of the examples in $S$.

\medskip

A similar link exists between the fractional clique number of $G_m(\H)$ 
and the communication complexity theoretic characterization of pure private learnability by~\cite{FeldmanX15}.
The latter is based on a communication game between two players whom we call Alice and Bob.
In the game, Alice's input is a hypothesis $h\in \H$, Bob's input is a point $x$, and Alice sends a \underline{single} message to Bob.
Their goal is that with probability at least $3/4$, Bob will be able to decode $h(x)$ from Alice's message.
Feldman and Xiao showed that the optimal number of bits required to perform this task is proportional to the representation dimension of $\H$,
and hence characterizes pure DP learnability. 

Fractional cliques are linked to hard distributions for the above communication game.
A fractional clique $\delta$ corresponds to a distribution $\nu$ over realizable datasets $S$ of size $m$
such that for every hypothesis~$h$:
\begin{equation}\label{eq:fc}
	\Pr_{S\sim \nu}\Bigl[h \text{ is consistent with $S$}\Bigr]\leq \frac{1}{\lvert \delta\rvert}.
\end{equation}
Consider a variant of the above equation where the event ``$h$ is consistent with $S$'' 
is replaced by the event ``$h$ classifies correctly at least $3/4$ of the examples in $S$''.
This variant induces a hard distribution for the communication game as follows.
Pick a random dataset $S\sim \nu$ and a concept $h\in \H$ such that $h$ is consistent with $S$.
Let Alice's input be $h$ and Bob's input be a random (unlabeled) example from $S$.
Standard arguments in communication complexity show that this distribution is indeed hard for the communication game.

To summarize the discussion, LP duality implies that the fractional clique and chromatic numbers are equal. 
When considering this along with the aforementioned connections, it illuminates a dual relationship 
between the representation dimension and the communication-complexity-based characterization of pure private learnability.

\subsubsection{Theorems~\ref{thm:PDP-FCD} and~\ref{thm:CD-DP}: Technical Overview}
Our proof of Theorems~\ref{thm:PDP-FCD} and~\ref{thm:CD-DP} rely on the exponential-polynomial dichotomies (Theorems~\ref{thm:SSP-FCD}, and~\ref{thm:SSP-CD}).

We begin with overviewing the proof of Theorem~\ref{thm:CD-DP}. By~\cite{AlonBLMM22}, a class $\H$ is approximately DP learnable if and only if it has a finite Littlestone dimension. Thus, it suffices to show that the clique dimension is finite if and only if the Littlestone dimension is finite.
Our proof yields explicit bounds of 
\[\LD(\H)\leq \CD(\H)\leq O(\LD(\H)\log \LD(\H)),\]
where $\LD(\H)$ denotes the Littlestone dimension of $\H$.
One direction is straightforward. If the Littlestone dimension is at least $m$, then there is a mistake tree of depth $m$ that is shattered by $\H$. Each branch of the tree corresponds to a realizable dataset of length $m$, and the collection of $2^m$ datasets correspond to the branches form a clique. Indeed, every two datasets disagree on the example corresponding to the least common ancestor of their branches.


The converse direction is more challenging. Assume that the clique dimension $\CD(\H)$ 
is at least~$m$. We prove the existence of a shattered mistake tree of depth $\tilde\Omega(m)$.
Like in Theorem~\ref{thm:SSP-CD}, we prove this using the existence of a balanced point (see~\Cref{eq:1}).
Indeed, given a clique $\delta$ of size $2^m$ in $G_m(\H)$, we pick the root of the mistake tree to be a balanced point $x$ in $\delta$, 
and proceed to find a balanced point in each of $\{S\in\delta : (x,0)\in S\}$ and $\{S\in\delta : (x,1)\in S\}$, corresponding to the left and right subtrees. We continue this way until one of the branches is consistent with exactly one dataset in $\delta$.
A standard calculation shows that this way one obtains a Littlestone tree of depth at least $\tilde\Omega(m)$.

We note that there is a simple example of a class $\H$ that satisfies $\LD(\H)<\CD(\H)$ (see \Cref{sec:future_work}).
We leave as an open question to determine whether $\LD(\H)= \Theta(\CD(\H))$.

\medskip

We now move to overview the proof of Theorem~\ref{thm:PDP-FCD}. Our proof follows by relating the representaion dimension with the fractional clique dimension, showing that one is finite if and only if the other is finite. This finishes the proof because the representation dimension characterizes pure DP learnability~\cite{Beimel19Pure}.
The direction showing that finite representation dimension implies finite fractional clique dimension follows directly by results from~\cite{Beimel19Pure}. In particular,~\cite{Beimel19Pure} proved a boosting result which reduces the probability of error
of the random class from $1/4$ to $\epsilon$. Applying their result with $\epsilon<1/m$ yields a bound on the fractional clique number.

For the converse direction, we use Theorem~\ref{thm:SSP-FCD}.
Assume that the fractional clique dimension of~$\H$ is finite;
therefore, for every $m$ the fractional clique number is bounded by $\mathtt{poly}(m)$.
Thus, by LP duality there exists a fractional coloring $c$ with $\poly(m)$ many colors.
Now, since fractional colorings correspond to distributions $\mu$ over hypotheses satisfying \Cref{eq:cp},
we can define a distribution over hypothesis classes by sampling $\mathtt{poly}(m)$ independent hypotheses from $c$.
By \Cref{eq:cp} it follows that for every realizable dataset $S$ of size $m$,
one of the $\mathtt{poly}(m)$ hypotheses is consistent with $S$ with probability at least a constant (say $1/4$).
By a standard generalization argument, this yields the desired bound on the representation dimension.

\subsection{Strong Duality in Infinite Contradiction Graphs}

Our proofs heavily rely on the equality of the fractional clique and chromatic numbers.
For finite graphs, this equality is a consequence of LP duality.
However, in learning theory, we often study infinite hypothesis classes $\H$
whose contradiction graphs are therefore also infinite.
In general, LP duality does not apply in infinite dimensional spaces.
Therefore, we prove the next theorem showing that any (possibly infinite) contradiction graph $G_m(\H)$
satisfies that its fractional clique and chromatic numbers are equal and are bounded from above by $2^m$.
\begin{theorem}\label{thm:minimax}
	Let $\X$ be an arbitrary domain, $\H\subseteq\{0,1\}^\X$ a concept class and $m\in\mathbb{N}$.
	Let $\omega_m^\star$ and $\chi_m^\star$ denote the fractional clique and chromatic numbers of the contradiction graph $G_m(\H)$.
	Then, 
	\[\omega_m^\star = \chi_m^\star\leq 2^m.\] 
	Moreover, there exists a fractional coloring realizing $\chi_m^\star$.
	(I.e.\ the infimum is in fact a minimum.)
\end{theorem}

The proof of Theorem~\ref{thm:minimax} uses tools from functional analysis, topology, and measure theory.
The equality $\omega_m^\star = \chi_m^\star$ is derived using Sion's Theorem~\cite{zbMATH03133049},
and the upper bound of $2^m$ hinges on Kolmogorov's Extension Theorem (see, e.g.\ Theorem 2.4.3 in \cite{tao2011introduction}). This proof appears in \Cref{app:minimax}, 
which we attempted to present in a manner accessible for readers who may only have basic familiarity with topology and functional analysis.





\section{Preliminaries}\label{sec:definitions}

\subsection{Learning}\label{sec:preliminaries_learning}

\paragraph{PAC Learning.}
We use standard notations from statistical learning; for more details see e.g. \cite{shalev2014understanding}.
 Let $\X$ be a domain; for simplicity, in this work, we assume that $\X$ is countable, although our arguments apply more generally. Given an hypothesis $h:\X\to \{0,1\}$, the \emph{empirical loss} of $h$ with respect to a dataset $S=\bigl((x_1,y_1),\ldots,(x_m,y_m)\bigr)$ is defined as
$\L_\S(h)\coloneqq\frac{1}{m}\sum_{i=1}^m\Ind[h(x_i)\neq y_i]$.
	We say that $h$ is \emph{consistent} with $S$ if $\L_\S(h)=0$.
	The \emph{population loss} of $h$ with respect to a distribution $\D$ over $\X\times\{0,1\}$ is defined as
	$\L_\D(h)\coloneqq\Pr_{(x,y)\sim \D}[h(x)\neq y]$.
	A distribution $\D$ over labeled examples is \emph{realizable} with respect to $\H$ if 
    $\inf_{h\in\H}\L_\D(h)=0$.
	
	For a set $Z$, let $Z^\star= \cup_{n=0}^\infty Z^n$. A \emph{learning rule} $\A$ is a (possibly randomized) algorithm that takes as input a dataset $S \in (\X \times \{0,1\})^\star$ and outputs a hypothesis~${h=\A(S)\in \{0,1\}^\X}$. 
	In the PAC learning model, the input $S$ is sampled i.i.d. from a realizable distribution~$\D$, and the learner's goal is to output an hypothesis with small population loss with respect to $\D$.
	More precisely, let $m, \alpha, \beta >0$. 
	We say that an algorithm $\A$ is an $(m,\alpha,\beta)$-learner for $\H$ if for every realizable distribution $\D$, 
	$\Pr_{S \sim \D^m}\left[ \L_\D(\A(S))>\alpha\right]<\beta$.
	Here, $\alpha$ is called the \emph{error}, $\beta$ is the \emph{confidence parameter}, and $m$ is the \emph{sample complexity}.
	A class $\H$ is \emph{PAC learnable} if there exists vanishing~$\alpha(m),\beta(m)\to 0$
	and an algorithm $\A$ such that for all $m$, algorithm $\A$ is a $(m,\alpha(m),\beta(m))$-learner for $\H$.

	\paragraph{Differential Privacy.}
	
	We use standard notations from differential privacy literature; for more details see e.g. \cite{DR14, Vadhan17survey}.
	A randomized learning algorithm $\A$ is \emph{differentially private} with parameters $(\epsilon, \delta)$, if for every input datasets $S, S' \in (\X\times\{0,1\})^m$ that differ on a single example, and every event~$E\subseteq \{0,1\}^\X$:
	\[\Pr[\A(S)\in E]\leq e^\epsilon\Pr[\A(S')\in E]+\delta,\]
	where the probability is over the randomness of $\A$.
	The parameters $\epsilon$ and $\delta$ are usually treated as follows: 
	$\epsilon$ is a small constant (say $\leq 0.1$) and $\delta$ is negligible, $\delta=m^{-\omega(1)}$ where $m$ is the input datasets size. When $\delta=0$ we say that $\A$ is \emph{pure differentially private}, and when $\delta>0$ we say that $\A$ is \emph{approximate differentially private}.
	
	An hypothesis class $\H$ is \emph{pure privately learnable} (abbreviated \emph{pure DP learnable}) if it is privately learnable by an algorithm which is $(\epsilon(m),0)$-differentially private, where $\epsilon(m)=O(1)$ is a numerical constant.
	An hypothesis class $\H$ is \emph{approximately privately learnable} (abbreviated \emph{DP learnable}) if it is PAC learnable by an algorithm $\A$ which is $(\epsilon(m),\delta(m))$-differentially private, where $\epsilon(m)=O(1)$ is a numerical constant and $\delta(m)=m^{-\omega(1)}$.

	\paragraph{Representation Dimension.}
	
	The representation dimension is a combinatorial parameter introduced by Beimel et al.~\cite{Beimel19Pure}  who used it to characterize pure DP learnability.
	
	\begin{definition}[Representation Dimension~\cite{Beimel19Pure}]\label{def:repdim}
		The \emph{representation dimension} of a concept class $\H$, denoted $\RepDim(\H)$, 
		is defined to be $\ln(d)$, where $d$ is the minimal integer for which there exists 
		a distribution $\P$ over hypothesis classes of size $d$ that satisfies the following.
		For every distribution $\D$ on labeled examples that is realizable by $\H$,
		\[\Pr_{\C\sim\P}\left[\text{$\exists h\in\C$ s.t.\ $L_{\D}(h)\leq\frac{1}{4}$} \right]\geq \frac{3}{4}.\]
	\end{definition}
	As~\cite{Beimel19Pure} show, the constants $1/4,3/4$ above can be replaced by any other pair of constants in $(0,1)$ without changing the semantics of the definition. This follows from the next lemma:
	
	\begin{lemma}[Boosting Probabilistic Representation, Lemma 18~\cite{Beimel19Pure}]\label{l:repdim_boost}
		Let $\H$ be a class with $\RepDim(\H)= d<\infty$. Then for every $0<\alpha, \beta<1$ there exists a probability distribution $\P$ over hypothesis classes of size $O\left((\frac{1}{\alpha})^{d+\ln\ln\ln(\frac{1}{\alpha})+\ln\ln(\frac{1}{\beta})}\right)$ which satisfies the following.
		For every realizable distribution $\D$ on labeled examples which is realizable by $\H$,
		\[\Pr_{\C\sim\P}\left[\text{$\exists h\in\C$ s.t.\ $L_{\D}(h)\leq\alpha$} \right]\geq 1-\beta.\]
	\end{lemma}
	
	\paragraph{Littlestone Dimension.}
	
	The Littlestone dimension is a combinatorial parameter that captures mistake and regret bounds in online learning \cite{Lit88, ben2009agnostic}. The definition of the Littlestone dimension uses the notion of \emph{mistake trees}. 
	A \emph{mistake tree} is a binary decision tree whose nodes are labeled with instances from~$\X$ and whose edges are labeled by $0$ or $1$ such that each internal node has one outgoing edge labeled $0$ and one outgoing edge labeled $1$. A root-to-leaf path in a mistake tree is a sequence of labeled examples $(x_1,y_1),\dots,(x_d,y_d)$. 
	The point $x_i$ is the label of the $i$'th internal node in the path, and $y_i$ is the label of its outgoing edge to the next node in the path.
	A class $\H$ \emph{shatters} a mistake tree if every root-to-leaf path is realizable by $\H$.
	The \emph{Littlestone dimension} of~$\H$, denoted $\LD(\H)$, is the largest number $d$ such that there exists a complete binary mistake tree of depth $d$ shattered by~$\H$.
	If $\H$ shatters arbitrarily deep mistake trees then we write $\LD(\H)=\infty$.
	 
	\subsection{Graph Theory}\label{sec:preliminaries_graph_theory}
	
	\paragraph{Cliques, Colorings, and Distributions.}
	
	Let $G=(V,E)$ be a (possibly countable) graph. 
	Denote by $\omega(G)$ the \emph{clique number} of~$G$, which is the largest size of a clique in~$G$.
	Denote by $\chi(G)$ the \emph{chromatic number} of $G$, which is the smallest number of colors needed to color the vertices of $G$ so that no two adjacent vertices share the same color.
	The clique and chromatic numbers have natural LP relaxations. 
	
	A \emph{fractional clique} is a function $\delta:V\to\R_{\geq 0}$ such that $\sum_{v \in I}{\delta(v)\leq 1}$ for every independent set~$I$.
    The \emph{size} of $\delta$ is $\lvert\delta\rvert\coloneqq\sum_{v \in V}{\delta(v)}$.
	The \emph{fractional clique number} of $G$, denoted~$\omega^\star(G)$, is defined by
	${\omega^\star(G)\coloneqq\sup_\delta\lvert\delta\rvert}$.
	A \emph{fractional coloring} is a finite measure\footnote{For a finite $G$, one typically defines a fractional coloring as a function $c:\I\to\mathbb{R}_{\geq 0}$ such that $\sum_{v\in I}c(I)\geq 1$ for every $v\in V$. The latter amounts to  $c$ being a discrete measure on~$\I$. In this work, we also consider infinite graphs and use this more general definition which allows for non-discrete measures on~$\I$. We refer the reader to Section~\ref{app:minimax} for a more detailed discussion.} $c$ on $\I$, where $\I$ is the family of all independent sets in $G$, such that 
	$c(\{I : v\in I\})\geq 1$ for every $v\in V$.
	The \emph{fractional chromatic number} of $G$, denoted~$\chi^\star(G)$, is defined by
	${\chi^\star(G)\coloneqq\inf_{c} c(\I)}$.
	Note that from (weak) LP-duality, for any graph $G$, $\omega(G)\leq\omega^\star(G)\leq\chi^\star(G)\leq\chi(G)$. In finite graphs strong LP-duality holds and  $\omega^\star(G)=\chi^\star(G)$. Theorem~\ref{thm:minimax} extends this equality to any (possibly infinite) contradiction graph.
	
	There is a natural correspondence between fractional colorings of a graph~$G$ and distributions over independent sets.
	Given a fractional coloring~$c$,
	normalizing~$c$ by the number of colores in~$c$, denoted by $\col(c)\coloneqq c(\I)$, induces a distribution~$\mu$ over independent sets, such that for every vertex $v$, 
	\[\Pr_{I\sim \mu}[v\in I]=\frac{1}{\col(c)}\cdot c(\{I:v\in I\}).\]
	Define the \emph{value} of a distribution $\mu$ over $\I$ to be ${\val(\mu)\coloneqq \inf_{v\in V}\Pr_{I\sim \mu}[v\in I]}$. 
	By taking infimum, 
	\[\val(\mu)\coloneqq \inf_{v\in V}\Pr_{I\sim \mu}[v\in I]=\frac{1}{\col(c)}.\]
	From the other direction, given a distribution $\mu$ over $\I$, normalizing $\mu$ by~$\val(\mu)$
	induces a fractional coloring of $G$.
	Minimizing the number of colors of a fractional coloring is equivalent to maximizing the value of the corresponding distribution, hence
	\[\frac{1}{\chi^\star(G)}=\sup_{\mu}\val(\mu)=\sup_{\mu}\inf_{v \in V}{\Pr_{I \sim \mu}\left[v \in I\right]}.\]
	
	Similarly, there is a correspondence between fractional cliques of a graph $G$ and distributions over vertices. 
	Normalizing a fractional clique~${\delta:V\to \R_{\geq0}}$ by~$\lvert\delta\rvert$, induces a distribution $\nu$ over vertices, and 
	normalizing a distribution $\nu$ over $V$ by 
	$\val(\nu)\coloneqq \sup_{I\in \I}\Pr_{v\sim \nu}[v\in I]$
	induces a fractional clique.
	Similarly, 
	\[\frac{1}{\omega^\star(G)}=\inf_{\nu}\val(\nu)=\inf_{\nu}\sup_{I \in \I}{\Pr_{v \sim \nu}\left[v \in I\right]}.\]
	For further reading about fractional graph theory see \cite{scheinerman13}.

	\subsection{The Contradiction Graph: Basic Facts}\label{sec:fundemental}
	
	In this section, we state basic lemmas about the structure of the contradiction graph.
	We begin with a discussion about the relation between interpolating learning rules and colorings of the contradiction graph. 
	Then, we discuss the relation between fractional colorings and cliques, and distributions over hypotheses and realizable datasets.    
	Omitted proofs can be found at \Cref{app:proofs}.
	
	\paragraph{Interpolating Algorithms are Proper Colorings.}
	There is a correspondence between proper colorings of the contradiction graph and interpolating learning rules. A deterministic learning rule $\A$ is said to be \emph{interpolating} with respect to a class $\H$ if for every realizable input dataset $S$, the output hypothesis~${h\coloneqq \A(S)}$ satisfies $h(x)=y$ for every labeled example ${(x,y)\in S}$. 
	A \emph{proper coloring} of a graph is an assignment of a color to each vertex so that no two adjacent vertices share the same color. In other words, a coloring is a partition of the vertices such that every subset in the partition is an independent set. 
	The correspondence between colorings and interpolation algorithms is a direct result of the following lemma which identifies independent sets in the contradiction graph with hypotheses\footnote{The mapping described in \Cref{l:colorings-algs} is $1-1$ for maximal independent sets.}.
	
	\begin{lemma}[Independent sets and consistent hypotheses]\label{l:colorings-algs}
		Let $\H$ be a class and $m$ be a natural number.
		\begin{enumerate}
			\item For every independent set $I$ in $G_m(\H)$, there exists an hypothesis ${h\in \{0,1\}^\X}$ such that $h$ is consistent with every dataset $S\in I$; i.e.\ for every dataset $S\in I$ and every example $(x,y)\in S$, we have $h(x)=y$.
			\item For every hypothesis $h$, the set of all datasets of size $m$ that are consistent with $h$ is an independent set in $G_m(\H)$; i.e.\ the set
			      \[V_{h}\coloneqq\{S=((x_1,y_1),\ldots, (x_m,y_m))\in V_m(\H)\mid \forall i, h(x_i)=y_i\}\]
			      is independent in $G_m(\H)$.
		\end{enumerate}
	\end{lemma}
	
	\noindent \textbf{Colorings $\to$ Algorithms.}
	Let $S$ be a realizable dataset of size $m$, and consider a coloring of $G_m(\H)$. By part 1 of \Cref{l:colorings-algs} there exists an hypothesis $h$ which is consistent with every dataset in $G_m(\H)$ colored with the same color as $S$. 
	Now simply define $\A(S)=h$. This defines an interpolating learning rule~$\A$ for~$\H$.  
	
	\noindent \textbf{Algorithms $\to$ Colorings.} 
	Given an interpolating learning rule $\A$, by part 2 of \Cref{l:colorings-algs} for every realizable dataset $S$ of size $m$, the set $V_{\A(S)}$ is independent. Note that since $\A$ is interpolating every dataset $S\in V_m(\H)$ is covered, indeed $S\in V_{\A(S)}$. Next, assign a unique color to all datasets in $V_{\A(S)}$. (If there is more than one possible color option for some datasets then arbitrarily choose one.) This defines a proper coloring of $G_m(\H)$. 
	
	\paragraph{Fractional Cliques and Colorings of the Contradiction Graph.} 
	
	In the contradiction graph independent sets correspond to hypothesis and vertices are realizable datasets. Therefore, fractional colorings can be viewed as distributions over hypotheses, and fractional cliques as distributions over realizable datasets:
	
	\begin{lemma}[Fractional cliques and colorings vs. distributions]\label{l:colorings_cliques_dists}
		Let $\H$ be a class, $m\in\Nat$. Then,
		\begin{enumerate}
			\item There exists a fractional coloring $c$ of $G_m(\H)$ with $\col(c)=\alpha>0$ if and only if there exists a distribution $\mu$ over hypotheses such that
			      \[\inf_{S}\Pr_{h\sim\mu}\left[\text{$h$ is consistent with $S$}\right]=\frac{1}{\alpha},\]
			      where the infimum is taken over realizable datasets of size $m$. \label{itm:colorings_dists}
			\item There exists a fractional clique $\delta$ of $G_m(\H)$ with $\lvert\delta\rvert=\alpha>0$ if and only if there exists a distribution $\nu$ over realizable datasets of size $m$ such that
			      \[\sup_{h}\Pr_{S\sim\nu}\left[\text{$h$ is consistent with $S$}\right]=\frac{1}{\alpha},\]
			      where the supremum is taken over hypotheses $h\in \{0,1\}^\X$.\label{itm:cliques_dists}
		\end{enumerate}
	\end{lemma}
	
	Let $\omega_m^\star$ and $\chi_m^\star$ denote the fractional clique and chromatic numbers of $G_m(\H)$. It follows that
	\begin{align}
		\frac{1}{\chi^\star_m}   & =\sup_{\mu} \inf_{\nu} \expect_{\substack{h\sim\mu\\S\sim\nu}}\left[\Ind[\text{$h$ is consistent with $S$}]\right],\label{eq:frac_chromatic} \\
		\frac{1}{\omega^\star_m} & =\inf_{\nu}\sup_{\mu}\expect_{\substack{h\sim\mu\\S\sim\nu}}\left[\Ind[\text{$h$ is consistent with $S$}]\right].\label{eq:frac_clique}      
	\end{align}
	where the supremum is taken over distributions over hypotheses, and the infimum is taken over distributions over realizable datasets of size $m$. (See discussion in \Cref{sec:preliminaries_graph_theory}.)

	\begin{corollary}\label{cor:optimal_frac_coloring}
		Let $\omega^\star_m$ denote the fractional clique number of $G_m(\H)$.
		Then there exists a distribution $\mu^\star$ over hypotheses such that for every realizable dataset $S$ of size $m$,
		\[\Pr_{h\sim\mu^\star}\left[\text{$h$ is consistent with $S$}\right]\geq\frac{1}{\omega^\star_m}.\] 
	\end{corollary}

	\begin{proof}
			
		By \Cref{thm:minimax} there exists a fractional coloring of value $\chi^\star_m=\omega^\star_m$, thus by \Cref{l:colorings_cliques_dists} there exists a distribution $\mu^\star$ over hypotheses as wanted.
	\end{proof}
	
	This identification between fractional colorings and distribution over hypotheses is useful, for example, for attaining bounds on $\omega^\star_m$. Given a class $\H$, in order to show that $\omega^\star_m\leq\alpha$ it is enough to find a distribution over $\{0,1\}^\X$ such that every realizable dataset is consistent with a random hypothesis with probability at least $1/\alpha$.
	
	One of the basic properties that $G_m(\H)$ satisfies is that the clique and fractional clique numbers are bounded by $2^m$.
	
	\begin{lemma*}[\Cref{l:CN_FCN_bounded} restatemtent]
		Let $\omega_m$ and $\omega_m^\star$ denote the clique and fractional clique numbers of $G_m(\H)$. Then, $\omega_m\leq\omega^\star_m\leq 2^m$.
	\end{lemma*}
	We note that \Cref{l:CN_FCN_bounded} holds in a more general setting as well, where the domain $\X$ is arbitrary (possibly uncountable).
	We refer the reader to \Cref{app:minimax} for more details.
	\begin{proof}[Proof of \Cref{l:CN_FCN_bounded}]
		Let $\delta$ be a fractional clique in $G_m(\H)$. 
		Draw a random $h \in \{0,1\}^\X$ such that $h(x)$ is drawn uniformly and independently from $\{0,1\}$ for each $x\in \X$. 
		Consider the random variable
		\[X=\sum_{S\in V_m(\H)}{\Ind\left[\text{$S$ is consistent with $h$}\right]\cdot \delta(S)}.\]
		Note that $X\leq1$ almost surely: indeed, the set ${\{S:S \text{ is consistent with } h\}}$ is independent for every hypothesis $h$.
		Therefore,
		\begin{align*}
			1\geq \mathbb{E}\left[X\right] & =\sum_{S\in V_m(\H)}{\Pr\left[S \text{ is consistent with } h\right]\cdot \delta(S)}                     \\
			                               & =\sum_{S\in V_m(\H)}2^{-m}\cdot \delta(S)  \\
			                               & =2^{-m}\lvert\delta\rvert.                                                                               
		\end{align*}
		Thus $\rvert\delta\lvert \leq 2^m$ and therefore $\omega^\star_m\leq 2^m$. Note that the inequality $\omega_m\leq\omega^\star_m$ trivially holds since every (integral) clique is a fractional clique. 
	\end{proof}

	\section{Approximate Privacy and Cliques}\label{sec:DPvsCD}
	
	In this Section, we prove \Cref{thm:CD-DP,thm:SSP-CD}.
    Recall, \Cref{thm:SSP-CD} states that the clique number of the contradiction graph obeys an exponential-polynomial dichotomy, and \Cref{thm:CD-DP} states that a class is approximately DP PAC learnable if and only if its clique dimension is finite.
	The main idea is to study the relations between the clique dimension and the Littlestone dimension.
	It is known that a class $\H$ is approximately DP learnable if and only if the Littlestone dimension $\LD(\H)$ is finite \cite{ALMM19,BunLM20}.
	Therefore, to prove \Cref{thm:CD-DP} it suffices to show that the clique dimension $\CD(\H)$ is finite if and only if $\LD(\H)$ is finite. 
	Recall that $\LD(\H)$ is equal to the depth of the deepest mistake tree that is shattered by $\H$.
	Observe that if a class $\H$ shatters a mistake tree of depth $m$, then the $2^m$ datasets corresponding to the leaves of the shattered tree form a clique in $G_m(\H)$. 
	This observation proves the following lemma.
	
	\begin{lemma}\label{l:LD_leq_CD}
		Let $\H$ be an hypothesis class. Then $\LD(\H)\leq\CD(\H)$.
	\end{lemma}
	
	In simple words, deep shattered trees imply large cliques. We will show that the opposite statement holds as well.
	The following lemma is the crux of our proof.
	It asserts that for any clique in $G_m(\H)$ 
	there exists an unlabeled data point $x$ which \emph{separates} a non-negligible fraction of the datasets in the clique.
	
	\begin{lemma}\label{l:balanced_example}
		Let $\H$ be an hypothesis class and let $C\subseteq V_m(\H)$ be a clique in $G_m(\H)$.
		Then, there exists an unlabeled data point $x\in \X$ 
		which is \emph{balanced} in the following sense:  
		at least $\frac{\lvert C\rvert-1}{2m}$ datasets in $C$ contain the labeled example $(x,1)$
		and at least $\frac{\lvert C\rvert -1}{2m}$ datasets in $C$ contain the labeled example $(x,0)$.
	\end{lemma}
	
	Before proving this lemma let us remark that quantitative improvements in the above lemma translate to improvements in \Cref{thm:SSP-CD} and to tighter bounds relating the Littlestone and clique dimensions. We elaborate on this in \Cref{sec:future_work}.

	\begin{proof}[Proof of \Cref{l:balanced_example}]
		Denote $c\coloneqq|C|$. Given a labeled example $(x,b)$ denote by $C_{(x,b)}$ the set of all datasets in $C$ that contains $(x,b)$:
		\[C_{(x,b)}:=\{S\in C \mid (x,b)\in S\}.\]
		We perform the following iterative process to find a balanced example $x$:
		\begin{tcolorbox}[colframe=gray!80!black, title=Eliminate Unbalanced Examples in Clique:]
			As long as there is $S\in C$ and $(x,b)\in S$ such that $\lvert C_{(x,1-b)}\rvert <\frac{c-1}{2m}$, do:
			\begin{enumerate}
				\item Update $S\rightarrow S\smallsetminus\{(x,b)\}$.
				\item Update the edges of $C$ accordingly: if a dataset $S'\in C_{(x,1-b)}$ does not contradict $S$ anymore (i.e.\ $S$ and $S'$ disagreed only $x$), then delete the edge between them.
			\end{enumerate}
		\end{tcolorbox}
		Observe that:
		\begin{enumerate}
			\item The number of iterations is at most $cm$: 
			      Consider the sum $\sum_{S\in C}|S|$. 
			      At the beginning of the process $\sum_{S\in C}\vert S\rvert=c m$,
			      and at each iteration, the size of one dataset is reduced by one
			      and hence $\sum_{S\in C}|S|$ decreases by one. 
			      The bound follows since $\sum_{S\in C}|S|$ is always non-negative.
			\item At each iteration, the number of edges that are deleted is less than $\frac{c-1}{2m}$.
			\item Thus, the total number of edges that are deleted during the process is less than $c m\cdot\frac{c-1}{2m}= {\binom{c}{2}} $. 
		\end{enumerate}
		Therefore, at the end of the process, there is at least one remaining edge. 
		That is, there are two datasets in $C$ that contradict each other on an unlabeled example $x$ 
		which satisfies both $| C_{(x,1)}|$ and $|C_{(x,0)}|$ are at least  $\frac{c-1}{2m}$, as required. 
	\end{proof}
	
	\begin{lemma}\label{l:clique_number_bounded}
		Let $\H$ be an hypothesis class and denote by $\omega_m$ the clique number of $G_m(\H)$. 
		Then, for all $m$,
		\[\omega_m\leq (2m+1)^{\LD(\H)}\leq (2m+1)^{\CD(\H)}.\]
	\end{lemma}
	
	The idea of the proof is to apply \Cref{l:balanced_example} on a maximal clique in $G_m(\H)$ and construct inductively a shattered Littlestone tree whose depth is large (as a function of~$\omega_m$).
	Together with the fact that the depth is at most $\LD(\H)$, the desired bound follows.
	
	Note that proving this lemma completes the proof of \Cref{thm:SSP-CD}: if $\LD(\H)=\infty$ then by previous observation $\omega_m=2^m$ for every $m$,
    and if $\LD(\H)$ is finite then by the above lemma $\omega_m\leq (2m+1)^{\LD(\H)}=\poly(m)$ for every $m$. 
	
	\begin{proof}[Proof of \Cref{l:clique_number_bounded}]
		Let $C$ be a clique of size $\omega_m$ in $G_m(\H)$. 
		Without loss of generality, we can assume $\omega_m>2m$ 
		(otherwise the bound trivially holds\footnote{If $\omega_m\leq 2m$ and $\LD(\H)\geq 1$, 
			then the inequality in the lemma holds. 
			In the degenerate case when $\LD(\H)=0$, we have $\lvert\H\rvert=1$ hence $\omega_m=1$ for all $m$, and the inequality holds as well.}).
		By \Cref{l:balanced_example} there exists an unlabeled example $x$ 
		such that each of the sets
		\begin{align*}
			R & = \{S\in C\mid (x,1)\in S\}, \\
			L & = \{S\in C\mid (x,0)\in S\}  
		\end{align*} 
		has size at least $\frac{\omega_m-1}{2m}\geq\frac{\omega_m}{2m+1}$. 
		Take $x$ to be the root of a mistake tree, 
		and recursively repeat the same operation on the sub-cliques induced by $R,L$. 
		This way, in the $i$'th step we have $2^i$ cliques and the size of each clique is at least 
		\begin{equation*}
			\frac{\omega_m}{(2m+1)^i}.   
		\end{equation*}
		Say this process terminates after $T$ steps, 
		yielding a shattered tree of depth~${T\leq \LD(\H)}$. 
		Note that the process terminates if and only if at least one of the produced cliques has size~$\leq 1$, hence
		\[\frac{\omega_m}{(2m+1)^T} \leq 1, \]
		which implies 
		\[\omega_m\leq (2m+1)^T\leq  (2m+1)^{\LD(\H)}.\] 
   The second inequality holds by \Cref{l:LD_leq_CD}.
	\end{proof}
	
	As an immediate corollary we conclude that for every class $\H$, the clique dimension is finite if and only if the Littlestone dimension is finite.
	
	\begin{corollary}
		Let $\H$ be an hypothesis class. Then 
		\[\CD(\H)<\infty\iff\LD(\H)<\infty\]
	\end{corollary}
	
	As shown earlier, this completes the proof of \Cref{thm:CD-DP}.
     We turn to state another quantitative relation between the Littlestone and clique dimensions that follows from \Cref{l:clique_number_bounded}.
	
	\begin{lemma}\label{l:CD_bound}
		Let $\H$ be an hypothesis class. Then
		\[\CD(\H)\leq \max\bigl\{2\LD(\H) \log(\LD(\H)), 300\bigr\}.\]
	\end{lemma}

    The proof of this lemma is straightforward and technical thus deferred to \Cref{app:proofs}.

	\section{Pure Privacy and Fractional Cliques}\label{sec:PDPvsFCD} 
	
	In this section we prove \Cref{thm:PDP-FCD,thm:SSP-FCD}.
    Recall, \Cref{thm:SSP-FCD} states that the fractional clique number of the contradiction graph obeys an exponential-polynomial dichotomy, and \Cref{thm:PDP-FCD} states that a class is pure DP PAC learnable if and only if its fractional clique dimension is finite.
	Beimel et al. proved that a class $\H$ is pure DP learnable if and only if the representation dimension $\RepDim(\H)$ is finite \cite{beimel2013characterizing}.
	Recall, $\RepDim(\H)$ is finite if there exists a distribution over finite hypothesis classes such that for every realizable distribution over labeled examples, with high probability, a random class contains an hypothesis that has small population error.
	We will show that $\RepDim(\H)$ is finite if and only if the fractional clique dimension $\FCD(\H)$ is finite. 
	We begin with a technical lemma. 
	An optimal fractional coloring of $G_m(\H)$ induces a distribution over hypotheses such that every realizable dataset of size $m$ is consistent with a random hypothesis
	with probability of at least $1/\omega^\star_m$ 
	(\Cref{cor:optimal_frac_coloring}).
	In other words, the probability that a random hypothesis has zero empirical loss (with respect to any realizable dataset of size $m$) is bounded from below by a positive constant.
	Using measure concentration arguments gives similar results when considering the population loss of a random hypothesis instead of the empirical loss. That is, there exists a distribution over hypotheses such that for every realizable distribution $\D$, the probability that
	a random hypothesis has small population loss  (with respect to $\D$) is bounded from below.
	This fact is essential for the proof of the equivalence between the representation dimension and the fractional clique dimension.
	
	\begin{lemma}\label{l:small_pop_err}
		Let $\H$ be a class and $m\in\Nat$. Then there exists a distribution $\mu^\star$ over hypotheses which satisfies the following. For every distribution $\D$ over labeled examples which is realizable by $\H$, and for every $0\leq\theta\leq 1$,
		\begin{align*}
			\Pr_{h\sim \mu^\star}\left[\L_\D(h)\leq\theta\right]
			  & \geq\frac{1}{\omega^\star_m}-(1-\theta)^m, 
		\end{align*}
		where $\omega^\star_m$ is the fractional clique number of $G_m(\H)$.
	\end{lemma}
	
	
	\begin{proof}[Proof of \Cref{l:small_pop_err}]
		The proof idea is to consider an optimal fractional coloring of $G_m(\H)$ (which exists by \Cref{thm:minimax}). 
		The number of colors of this coloring is the fractional chromatic number, which equals to the fractional clique number.
		Fractional colorings correspond to distributions over hypotheses, as demonstrated by \Cref{l:colorings_cliques_dists}. 
		This optimal fractional coloring translates to an optimal distribution $\mu^\star$ over hypotheses.
		Then, we use a concentration argument to reason that this distribution achieves the desired bound.
			
		Let $\mu^\star$ be, as in~\Cref{cor:optimal_frac_coloring}, a distribution over hypotheses which satisfied the following: for every realizable dataset $S$ of size $m$,
		\begin{equation*}
			\Pr_{h \sim \mu^\star}\left[\text{$h$ is consistent with $S$}\right]\geq\frac{1}{\omega^\star_m}.\label{eq:by_cor_optimal_frac_coloring}
		\end{equation*}
		Since this is true for every realizable dataset $S$ of size $m$, it also holds on expectation when sampling $S$ from a realizable distribution.
		Let $\D$ be a realizable distribution over labeled examples.
		Hence,
		\begin{align*}
			\frac{1}{\omega^\star_m} & \leq \expect_{S\sim \D^m}\left[\expect_{h\sim \mu^\star}\left[\Ind_{\{\text{$h$ is consistent with $S$}\}}|S\right]\right]\notag \\
			                         & =\expect_h\left[\expect_S\left[\Ind_{\{\text{$h$ is consistent with $S$}\}}|h\right]\right]\notag                                \\
			                         & =\expect_h\left[(1-\L_\D(h))^m\right]. \label{eq:E_h,E_S}                                                                         
		\end{align*}
		Denote a random variable~$X=\bigl(1-\L_\D(h)\bigr)^m$, where $h\sim \mu^\star$. 
		Thus, for every $\theta\in [0,1]$,
		\begin{align*}
			\expect_{h\sim \mu^\star}\left[X\right] & \leq (1-\theta)^m\cdot\Pr\left[X<(1-\theta)^m\right]+1\cdot\Pr\left[X\geq(1-\theta)^m\right] \tag{$X\leq 1$ almost surely} \\
			                                        & \leq(1-\theta)^m+\Pr\left[X\geq(1-\theta)^m\right]                                                                         \\
			                                        & = (1-\theta)^m + \Pr_{h\sim \mu^\star}\bigl[\L_\D(h) \leq \theta\bigr].  
		\end{align*}
		Hence,
		\begin{align*}
			\Pr_{h\sim \mu}\bigl[\L_\D(h)\leq \theta\bigr]
			  & \geq \expect[X] - (1-\theta)^m               \\ 
			  & \geq \frac{1}{\omega^\star_m} -(1-\theta)^m. 
		\end{align*}
			
	\end{proof}
	
	We now turn to prove \Cref{thm:PDP-FCD}.
	As elaborated above, it suffices to show that for every class $\H$, the dimension $\RepDim(\H)$ is finite if and only if $\FCD(\H)$ is finite.
	A key tool used in the proof is the SSP lemma for the fractional clique dimension, which states that $\FCD(\H)$ is finite if and only if for every $m$, the fraction clique number of $G_m(\H)$ is bounded by a polynomial in $m$ (\Cref{thm:SSP-FCD}). 
	We begin by proving \Cref{thm:PDP-FCD} assuming \Cref{thm:SSP-FCD},
	and then we will formally prove \Cref{thm:SSP-FCD}. 
	
	\begin{lemma}[Finite fractional clique dimension $\to$ finite representation dimension]
		Let $\H$ be an hypothesis class and assume $\FCD(\H)=d<\infty.$
		Denote $\epsilon=\frac{1}{\omega^\star_{d+1}}-\frac{1}{2^{d+1}}$ where $\omega^\star_{d+1}$ is the fractional chromatic number of $G_{d+1}(\H)$.
		Then, 
		\[\RepDim(\H)=O\left(\frac{\log\frac{1}{\epsilon}}{\epsilon^2}\cdot\ln\left(\frac{\log\frac{1}{\epsilon}}{\epsilon^2} \right)+\frac{\log\frac{1}{\epsilon}}{\epsilon^2}\cdot\ln\ln\left(\frac{\log\frac{1}{\epsilon}}{\epsilon^2} \right)\right).\]
	\end{lemma}
	
	\begin{proof} 
		The goal is to construct a distribution $\P$ over finite hypothesis classes such that for every realizable distribution over labeled examples $\D$:
		\begin{equation}\label{eq:good_dist_over_classes}
			\Pr_{\C\sim \P}\Bigl[(\exists h\in \C): \L_\D(h)\leq \frac{1}{4}\Bigr] \geq \frac{3}{4}.
		\end{equation}   
		By \Cref{thm:SSP-FCD} there exists a natural number $\alpha=O\left(\log\left(\frac{1}{\epsilon}\right)/\epsilon^2\right)$ such that for every~$m$, we have~$\omega^\star_m\leq m^\alpha$.
		Let $m=m(\alpha)$ to be determined later on. 
		By applying \Cref{l:small_pop_err} with $\theta=\frac{1}{4}$, 
		it follows that there exists a distribution $\mu^\star$ over hypotheses such that for every realizable distribution $\D$,
		\begin{equation}\label{eq:ld<quarter}
			\Pr_{h\sim \mu^\star}\left[\L_\D(h)\leq\frac{1}{4}\right]\geq\frac{1}{m^\alpha}-\left(\frac{3}{4}\right)^m =:q(m).
		\end{equation} 
		Sampling $k$ hypotheses $h_1,\ldots,h_k$ i.i.d. from $\mu^\star$ induces a probability distribution over classes of size at most $k$. 
		Denote this distribution by $\P_k$.
		Observe that
		\begin{align*}
			\Pr_{\C\sim\P_k}\left[\text{$\exists h\in \C$, $L_{\D}(h)\leq\frac{1}{4}$} \right] & = 1-\Pr_{\{h_i\}_{i=1}^k\sim(\mu^\star)^k}\left[       
			(\forall i): L_{\D}(h_i)>\frac{1}{4}\right] \\
			                                                                                   & \geq1-\bigl(1-q(m)\bigr)^k. \tag{by \Cref{eq:ld<quarter}} 
		\end{align*}
		The following technical lemma concludes the proof:
		\begin{lemma}\label{l:tech_fin_FCD_implies_fin_repdim}
			Let $\alpha\geq2$, and set $m=\lfloor20\alpha\ln \alpha\rfloor$ and $k=4m^\alpha$. Then,
			\[\bigl(1-q(m)\bigr)^{k}\leq \frac{1}{4}.\]
		\end{lemma}
		The proof of \Cref{l:tech_fin_FCD_implies_fin_repdim} is deferred \Cref{app:proofs}.
		As a result, by setting $m,k$ as in \Cref{l:tech_fin_FCD_implies_fin_repdim}, the distribution $\P_k$ satisfies the property described in \Cref{eq:good_dist_over_classes} and therefore  $\RepDim(\H)=O(\ln k)=O(\alpha\ln \alpha+ \alpha\ln\ln \alpha) < \infty$.

			
	\end{proof}
	
	\begin{lemma}[Finite representation dimension $\to$ finite fractional clique dimension]
		Let $\H$ be a class and assume $\RepDim(\H) =d<\infty$. Then $\FCD(\H)<\infty$.
			    
	\end{lemma}
	
	\begin{proof}
		We need to show that there exists a natural number $m$ such that the fractional clique number of $G_m(\H)$ is strictly smaller than $2^m$.
		By \Cref{l:colorings_cliques_dists} it is enough to show that there exist $m$, a distribution $\mu$  over hypotheses, and $\epsilon >0$ such that for every realizable dataset $S$ of size $m$,
		\begin{equation}\label{eq:prob_h_const}
			\Pr_{h\sim\mu}\left[\text{$h$ is consistent with $S$}\right]\geq\frac{1}{2^m}+\epsilon.
		\end{equation}
		By \Cref{l:repdim_boost}, there exists a distribution $\P$ over classes of size $k=O((2m)^{d+\ln\ln\ln(2m)+\ln\ln(4)})$ 
		such that for every realizable distribution~$\D$,
		\begin{equation}\label{eq:prob_repdim}
			\Pr_{\C\sim\P}\left[\text{$\exists h\in\C$ s.t.\ $L_{\D}(h)\leq\frac{1}{2m}$} \right]\geq \frac{3}{4}.
		\end{equation}
		Randomly sampling $\C\sim\P$ and then randomly uniformly sampling $h\in C$ induces a probability distribution over hypotheses. Denote this distribution by $\mu$.
		It suffices to show that $\mu$ satisfies the property in \Cref{eq:prob_h_const}.
		Let ${S=((x_1,y_1),\ldots,(x_m,y_m))}$ be a realizable dataset of size $m$. Define a distribution over labeled examples $\D$ as follows 
		\[\D(x,y)=\frac{1}{m}\cdot \sum_{i=1}^{m}{\Ind[(x,y)=(x_i,y_i)]}.\]
		Note that $\D$ is realizable with respect to $\H$ since $S$ is a realizable dataset. 
		Furthermore, for every hypothesis $h$, 
		\[\L_\D(h)=\frac{1}{m}\sum_{i=1}^{m}{\Ind[h(x_i)\neq y_i]}=\L_\S(h). \tag{ by definition of $\D$}\] 
		Therefore, $h$ is consistent with $S$ if and only if $\L_\D(h)<\frac{1}{m}$.
		Therefore,
		\begin{align*}
			\Pr_{h\sim\mu}\left[\text{$h$ is consistent with $S$}\right] & \geq \Pr_{h\sim\mu}\left[\L_\D(h)\leq\frac{1}{2m}\right]                                                                                 \\ 
			                                                             & \geq\Pr_{\C\sim\P}\left[\text{$\exists h\in\C$ s.t.\ $L_{\D}(h)\leq\frac{1}{2m}$ } \right]\cdot\frac{1}{k} \tag{by definition of $\mu$} \\
			                                                             & \geq \frac{3}{4k} \tag{by \Cref{eq:prob_repdim}}                                                                                        \\
			                                                             & \geq\frac{1}{2^m} + \epsilon,                                                                                                             
		\end{align*}
		where the last inequality holds with some $\epsilon > 0$ provided that $m>\log(\frac{4k}{3})$.
		Thus, it remains to show that there exists $m$ such that $m>\log(\frac{4k}{3})$.
		(Recall that $k=k(m)$ is the size of the hypothesis classes in the support of the distribution $\P$ such that \Cref{eq:prob_repdim} holds.) 
		Now, since $k=O\left((2m)^{d+\ln\ln\ln(2m)+\ln\ln(4)}\right)$
		(\Cref{l:repdim_boost}), we get that $\log(\frac{4k}{3})$ is sublinear in $m$.
		Thus, $m\geq \log(\frac{4k}{3})$  for a large enough $m$, as required.
			
	\end{proof}
	
	\subsection{SSP Lemma for Fractional Clique Dimension}
    We turn to prove \Cref{thm:SSP-FCD}.
	We begin with a general layout of the proof.
	Let $\H$ be a class and assume that \Cref{item:ssp-fcd1} does not hold. Meaning, there exists a natural number $m_0$ such that the fractional chromatic number of $G_{m_0}(\H)$ is strictly smaller than $2^{m_0}$.
	In order to show that $\omega^\star_m$ is bounded by a polynomial, by \Cref{l:colorings_cliques_dists} it is enough to show that for every $m$, there exists a distribution over hypotheses $\mu_m$,
	such that for every realizable dataset $S$ of size~$m$, a random hypothesis is consistent with $S$ with probability at least $\poly(m^{-1})$; i.e. for every realizable dataset of size $m$,
	\begin{equation}\label{eq:prop_of_good_dist}
		\Pr_{h\sim\mu_m}\left[\text{$h$ is consistent with $S$}\right]\geq \frac{1}{\mathsf{poly}(m)}.
	\end{equation}
	Specifically, we will derive an explicit upper bound on the degree of the polynomial on the right hand side, denoted by $\alpha$:
	\begin{equation}\label{eq:upper_bound_degree}
		\alpha=O\left(\frac{log\frac{1}{\epsilon}}{\epsilon^2}\right), 
	\end{equation}
	where $\epsilon=\frac{1}{\omega^\star_{m_0}}-\frac{1}{2^{m_0}}$.
	The idea is to use a boosting argument to show such distributions exist. The analysis involves a reduction to regret analysis for online predictions using experts' advice.
	\begin{itemize}
		\item[Step $1$:] There exists a distribution $\tilde{\mu}$ over hypotheses such that a random hypothesis is ``slightly better then a random guess". 
		      Formally, for every realizable distribution over labeled examples $\D$, and for every $\gamma\in (0,\frac{1}{2})$,
		      \[\Pr_{h\sim\tilde{\mu}}[\L_\D(h)\leq\frac{1}{2}-\gamma]>\epsilon-2\gamma.\]
		\item[Step $2$:] Set $T=\lceil\frac{2\log m}{\gamma^2}\rceil$. 
		      For every realizable dataset $S$ of size $m$, the majority vote of $T$ i.i.d. hypotheses sampled from $\tilde{\mu}$ (from Step $1$) is consistent with $S$, with probability at least $(\epsilon-2\gamma)^T=m^{-\frac{2}{\gamma^2}\log \frac{1}{\epsilon-2\gamma}}$.
		\item[Step $3$:] Denote by $\mu_m$ the a distribution over hypotheses induced by sampling~$T$ i.i.d. hypotheses from $\tilde{\mu}$ and then taking their majority vote.
		      Plug in $\gamma=\frac{\epsilon}{4}$ and conclude
		      \[\Pr_{h\sim\mu_m}\left[\text{$h$ is consistent with $S$}\right]\geq m^{-O\left(\frac{log\frac{1}{\epsilon}}{\epsilon^2}\right)},\]
		      which concludes the proof.
	\end{itemize}

	\begin{proof}[Proof of \Cref{thm:SSP-FCD}]
		From \Cref{l:small_pop_err} there exists a distribution over hypotheses $\tilde{\mu}$ such that for every realizable distribution $\D$
		\begin{align*}
			\Pr_{h\sim \tilde{\mu}}\left[\L_\D(h)\leq\frac{1}{2}-\gamma\right] & \geq \frac{1}{\omega_{m_0}^\star}-\left(\frac{1}{2}+\gamma\right)^{m_0} \\
			                                                                  & =\frac{1}{2^{m_0}}+\epsilon-\left(\frac{1}{2}+\gamma\right)^{m_0}         
		\end{align*}
		where $0<\gamma<\frac{1}{2}$ will be determined later on.
		Observe that since ${\left(\frac{1}{2}+\gamma\right)^{m_0}<\frac{1}{2^{m_0}}+2\gamma}$ for all $0<\gamma < \frac{1}{2}$\footnote{$\left(\frac{1}{2}+\gamma\right)^{m_0}=
			\frac{1}{2^{m_0}}+\gamma\sum_{k=0}^{m_0-1} {\binom{m_0}{k}} \gamma^{m_0-1-k}\cdot\frac{1}{2^k} <
			\frac{1}{2^{m_0}}+\gamma \frac{1}{2^{m_0-1}}\sum_{k=0}^{m_0-1} {\binom{m_0}{k}}=
			\frac{1}{2^{m_0}}+\gamma \frac{2^{m_0}-1}{2^{(m_0-1)}}<
			\frac{1}{2^{m_0}}+2\gamma $}, we get that
		\begin{equation}\label{eq:pos_prob_to_weak_learner}
			\Pr_{h\sim \tilde{\mu}}\left[\L_\D(h)\leq\frac{1}{2}-\gamma\right]\geq \epsilon -2\gamma.  
		\end{equation}
		Here and below, we say that an hypothesis $h$ is \emph{$\gamma$-good} with respect to a distribution $\D$ if $\L_\D(h)\leq\frac{1}{2}-\gamma$. 
		Let $m$ be a natural number and $S=\left\{(x_i, y_i)\right\}_{i=1}^{m}$ be a realizable dataset of size $m$. 
		Define $\mu=\mu_m$ to be the distribution over hypotheses induced by taking the majority vote of $T$ i.i.d.\ hypotheses sampled from $\tilde{\mu}$,
		where ${T=\lceil\frac{2\log m}{\gamma^2}\rceil}$.
			
		The goal is to prove that the distribution $\mu$ satisfies the property described in \Cref{eq:prop_of_good_dist}.
		The following two lemmas will complete the proof.

		\begin{lemma}\label{l:tech2}
			Let $S$ be a realizable dataset of size $m$ and $h_1,\ldots,h_T$ a sequence of hypotheses where $T=\lceil\frac{2\log m}{\gamma^2}\rceil$. Then, there exists a sequence of realizable distributions $\D_1, \D_2,\ldots, \D_T$ such that $\D_t$ is a function of $h_1,\ldots, h_{t-1}$ and the following condition holds. If for every $t\in [T]$, $h_t$ is $\gamma$-good with respect to $\D_t$,  then $S$ is consistent with $\operatorname{MAJ}\{h_t\}_{t=1}^T$.
		\end{lemma}
			
		The proof of \Cref{l:tech2} uses a regret analysis for online learning using expert advice, and is deferred after a short discussion introducing concepts and notations from online learning.
			
		\begin{lemma}\label{l:tech3}
			For every $t\leq T$,
			\[\Pr_{(h_1,\ldots, h_t)\sim\tilde{\mu}^t}\left[\text{$h_k$ is $\gamma$-good w.r.t.\ $\D_k$, } \forall k=1,\ldots ,t\right]\geq (\epsilon-2\gamma)^t,\]
			where $\D_1,\ldots, \D_t$ is a sequence of realizable distributions as in \Cref{l:tech2}.
		\end{lemma}
			
		The proof of \Cref{l:tech3} is technical and follows from  \Cref{eq:pos_prob_to_weak_learner} and simple induction.
			
		With \Cref{l:tech2} and \Cref{l:tech3} in hand, the following calculation finishes the proof. 
		\begin{align*}
			\Pr_{h\sim\mu}\left[\text{$h$ is consistent with $S$}\right] & =\Pr_{(h_1,\ldots, h_T)\sim\tilde{\mu}^T}\left[\text{$S$ is consistent with $\operatorname{MAJ}\{h_t\}_{t=1}^T$}\right] \tag{by the definition of $\mu$} \\
			                                                             & \geq                                                                                                                                                     
			\Pr_{(h_1,\ldots, h_T)\sim\tilde{\mu}^T}\left[\text{$h_t$ is $\gamma$-good w.r.t.\ $\D_t$, } \forall t=1,\ldots T\right] \tag{from \Cref{l:tech2}}\\
			                                                             & \geq(\epsilon-2\gamma)^T  \tag{from \Cref{l:tech3}}                                                                                                        \\
			                                                             & \geq (\epsilon-2\gamma)^{\frac{2\log m}{\gamma^2} \tag{by the choice of $T$}}                                                                              \\
			                                                             & =m^{-\frac{2}{\gamma^2}\log \frac{1}{\epsilon-2\gamma}}.                                                                                                   
		\end{align*}
			
		Denote $\alpha(\gamma)=\frac{2}{\gamma^2}\log \frac{1}{\epsilon-2\gamma}$. 
		The upper bound stated in \Cref{eq:upper_bound_degree} is obtained by plugging in $\gamma=\frac{\epsilon}{4}$.
	\end{proof}
	
	\begin{proof}[Proof of \Cref{l:tech3}]
		We will show that for $t\leq T$, 
		\begin{equation}\label{eq:prob_all_gamma_good}
			\Pr_{(h_1,\ldots, h_t)\sim\tilde{\mu}^t}\left[\text{$h_k$ is $\gamma$-good w.r.t.\ $\D_k$, } k=1,\ldots ,t\right]\geq (\epsilon-2\gamma)^t.
		\end{equation}
		Indeed, the base case where $t=1$ follows directly from \Cref{eq:pos_prob_to_weak_learner} and the fact that $\D_1$ does not depend on $h_1$. For $t>1$,
		\begin{flalign}
			&
			\Pr_{(h_1,\ldots, h_t)\sim\tilde{\mu}^t}\left[\text{$h_k$ is $\gamma$-good w.r.t.\ $\D_k$, } k=1,\ldots, t\right] =\nonumber\\
			&\quad =
			\expect_{(h_1,\ldots, h_t)\sim\tilde{\mu}^t}\left[
				\Ind_{\{\text{$h_k$ is $\gamma$-good w.r.t.\ $\D_k$, } k=1,\ldots, t-1\}}\cdot 
			\Ind_{\{\text{$h_t$ is $\gamma$-good w.r.t.\ $\D_t$}\}}\right] \nonumber \\
			& \quad = 
			\expect_{h_1,\ldots, h_{t-1}}\left[
				\expect_{h_t}\left[
					\Ind_{\{\text{$h_k$ is $\gamma$-good w.r.t.\ $\D_k$, } k=1,\ldots, t-1\}}\cdot 
					\Ind_{\{\text{$h_t$ is $\gamma$-good w.r.t.\ $\D_t$}\}}
				\mid h_1,\ldots, h_{t-1}\right]
			\right] \label{eq:induction1}\\
			& \quad = 
			\expect_{h_1,\ldots, h_{t-1}}\left[
				\Ind_{\{\text{$h_k$ is $\gamma$-good w.r.t.\ $\D_k$, } k=1,\ldots, ,t-1\}}\cdot
				\expect_{h_t}\left[
					\Ind_{\{\text{$h_t$ is $\gamma$-good w.r.t.\ $\D_t$}\}}
				\mid h_1,\ldots, h_{t-1}\right]
			\right] \label{eq:induction2}\\
			& \quad \geq  (\epsilon-2\gamma)\expect_{h_1,\ldots, h_{t-1}}\left[
			\Ind_{\{\text{$h_k$ is $\gamma$-good w.r.t.\ $\D_k$, } k=1,\ldots, t-1\}}\right] \label{eq:induction3}\\
			& \quad \geq
			(\epsilon-2\gamma)^t. \label{eq:induction4}
		\end{flalign}
		\Cref{eq:induction1} is obtained by applying the law of total expectation.
		\Cref{eq:induction2} holds since for every $k<t$, $\D_k$ does not depend on $h_t$.
		\Cref{eq:induction3} holds since for every realization of $h_1,\ldots, h_{t-1}$, $\D_t$ is determined and does not depend on $h_t$, hence by \Cref{eq:pos_prob_to_weak_learner},
		\[\expect_{h_t\sim\tilde{\mu}}\left[
				\Ind_{\{\text{$h_t$ is $\gamma$-good w.r.t.\ $\D_t$}\}}
			\mid h_1=\tilde{h}_1,\ldots, h_{t-1}=\tilde{h}_{t-1}\right]\geq \epsilon-2\gamma.\]
			Finally, \Cref{eq:induction4} is true by induction.
			\end{proof}
					
			In order to finish the proof of \Cref{thm:SSP-FCD} it is left to prove \Cref{l:tech2}. We begin with a short technical overview of online learning using experts' advice.
					 
			\paragraph{Learning Using Expert Advice.} 
			We briefly introduce the setting of online prediction using expert advice. 
			Let $Z=\{z_1,\ldots,z_m\}$ be a set of experts and $I$ be a set of instances. 
			Given an instance $i\in I$, an expert $z$ predicts a prediction $p=z(i)\in\{0,1\}$. 	
			
   In this online learning problem, at each round $t$ the learner receives an instance $i_t$ and has to come up with a prediction according to advice from the $m$ experts. The learner does that by obtaining at each round a distribution $w_t$ over the set of experts $Z$, randomly picking an expert $z^{(t)}\sim w_t$, and predicting $p_t=z^{(t)}(i_t)$.
			Then, the learner receives $f_t$, the true label of $i_t$,
			and pays a cost $l_t=\Ind[p_t\neq f_t]$ for taking the advise of $z^{(t)}$. Lastly, the learner updates the distribution over the set of experts $Z$ (according to some update rule).
			\begin{tcolorbox}[colframe=gray!80!black, title=Online predictions using expert advice-general scheme]
				\begin{itemize}
					\item []\textbf{Input:} $Z=\{z_1,\ldots, z_m\}$ a set of experts.
					\item []\textbf{Initialize:} distribution $w_1$ over $Z$. 
					\item [] For every round $t=1,2,\ldots,T$:
					      \begin{itemize}
					      	\item Receive an instance $i_t$. 
					      	\item Sample an expert $z^{(t)}\sim w_t$.
					      	\item Predict $p_t=z^{(t)}(i_t)$.
					      	\item Receive true label $f_t$ and pay cost of $l_t=\Ind[p_t\neq f_t]$.
					      	\item Update the distribution $w_t$ to $w_{t+1}$.
					      \end{itemize}
				\end{itemize}
			\end{tcolorbox}
			The \emph{total loss} of the learner is the expected sum of costs: 
			\[L(T)=\expect\left[\sum_{t=1}^T{l_t}\right].\]
			The \emph{regret} of the learner is the difference between the total loss and the loss of the best expert:
			\[\mathsf{Reg}(T)= L(T)-\min_{z_j\in Z}\sum_{t=1}^T\Ind[z_j(i_t)\neq f_t].\]
			The goal of the learner is to compete with the best expert in $Z$; i.e. the goal is to minimize the regret. 
			There are several well studied algorithms for online prediction using expert advice which achive sublinear regret. 
			A classic example is \emph{Multiplicative Weights} (see section 21.2 in \cite{shalev2014understanding}) which satisfy the following regret bound:
			\begin{equation}\label{eq:regret_bound}
				\mathsf{Reg}(T)\leq \sqrt{2T\log m}.
			\end{equation}
			We will use that fact to prove \Cref{l:tech2}.

			\begin{proof}[Proof of \Cref{l:tech2}]
				Let $S=\left\{z_j=(x_j, y_j)\right\}_{j=1}^{m}$ be a realizable dataset of size $m$ and let $h_1,\ldots, h_T$ be a sequence of hypotheses where $T=\lceil\frac{2\log m}{\gamma^2}\rceil$. Recall, an hypothesis $h$ is $\gamma$-good with respect to a distribution $\D$ if ${\L_\D(h)\leq\frac{1}{2}-\gamma}$. 
				The goal is to show that there exists a sequence of realizable distributions $\D_1,\ldots, \D_T$ which satisfies the following properties: 
				\begin{enumerate}
					\item $\D_t=\D_t(h_1,\ldots, h_{t-1})$ for all $t\leq T$. In particular, $\D_1$ does not depend on $h_1,\ldots, h_T$.
					\item If for every $t$, $h_t$ is  $\gamma$-good with respect to $\D_t$, then $S$ is consistent with the majority hypothesis $\operatorname{MAJ}\{h_t\}_{t=1}^T$.
				\end{enumerate}
							
				The idea of the proof is as follows. We think of the dataset $S$ as a set of $m$ ``experts". Then we simulate an online learner for learning using experts' advice with $h_1,\ldots, h_T$ as ``instances". 
				The distributions $\D_1,\ldots, \D_T$ will be induced by the internal weights of the ``experts" over time during the simulation. Regret analysis, together with the specific choice of number of rounds $T$, will lead to the desired results. 
							
				Let $\A$ be an online learner for prediction using expert advice, which satisfies the regret bound in \Cref{eq:regret_bound}, with the following settings:
				Our set of experts is the dataset $S=\left\{z_1,z_2,\ldots, z_m\right\}$, and a possible instance is an hypothesis $h\in\{0,1\}^\X$.
				Given an expert $z=(x,y)$ and an instance $h$, the loss of $z$ on $h$ is defined to be $l(z,h)=\Ind[h(x)=y]$; i.e.\ $\A$ suffers a loss if the hypothesis $h$ is consistent with the example $z=(x,y)$.
				Now, when running $\A$ on the input sequence $h_1,\ldots, h_T$, the algorithm maintains a distribution $w_t$ on the dataset $S$ at each timestamp $t\in [T]$.
				Denote $\D_t$ to be a distribution over all labeled examples induced by $w_t$; i.e,
				\[\D_t(x,y)=\sum_{j:z_j=(x,y)}{w_t(z_j)}\]
				Note that since $S$ is a realizable dataset, $\D_t$ is a realizable distribution for all $t$.
							
				We claim that the sequence $\D_1,\ldots, \D_T$ fulfills the desired properties. 
				Indeed, property (1) holds since the internal distribution obtained by the learner, $w_t$, does not depend on the current instance~$h_t$. 
				We next show that property (2) holds as well.
				assume $h_t$ is $\gamma$-good with respect to $\D_t$ for all $t$. Then, the expected loss of $\A$ at timestamp $t$ satisfies
				\begin{align*}
					\expect[l _t] & =\Pr_{(x,y)\sim \D_t}[h_t(x)=y] \\
					              & =1-L_{\D_t}(h_t)                \\
					              & \geq \frac{1}{2}+\gamma.        
				\end{align*}
				By linearity, the expected total loss of $\A$ is bounded from below by
				\begin{equation}\label{eq:tot_loss_bound}
					\expect\left[\sum_{t=1}^{T}{l_t}\right]\geq\frac{T}{2}+\gamma T.
				\end{equation}
				If by contradiction $S$ is not consistent with $\operatorname{MAJ}\{h_1,\ldots,h_T\}$, then there exists $j\in [m]$ such that 
				\begin{equation}\label{eq:maj_not_const}
					\sum_{t=1}^{T}{\Ind[h_t(x_j)=y_j]}=\sum_{t=1}^{T}{l(z_j,h_t)}<\frac{T}{2}.
				\end{equation}
				From the regret bound in \Cref{eq:regret_bound}, combining together \Cref{eq:tot_loss_bound,eq:maj_not_const} yields
				\[ \gamma T<\sqrt{2T\log m}.\]
				However, this is a contradiction to the choice of $T$ (recall $T=\lceil\frac{2\log m}{\gamma^2}\rceil$). Thus property (2) indeed holds and the proof of the lemma is complete.
			\end{proof}

			\section{Future Work and Open Questions}\label{sec:future_work}
			We conclude this manuscript with some open questions and suggestions for future work.
			\paragraph{Clique Dimension vs.\ Littlestone Dimension.}
					
			\Cref{l:LD_leq_CD,l:CD_bound} tie together the clique dimension and the Littlestone dimension of a class $\H$,
			showing that they are equivalent up to log factors.
			It is natural to ask whether a tighter relationship holds:
			\begin{question}\label{q:LDCD}
				Is it the case that $\CD(\H)= \Theta(\LD(\H))$?
			\end{question}
			Notice that $\LD(\H)\leq \CD(\H)$ for every $\H$, because the datasets corresponding to the branches of a shattered tree form a clique.
			Thus, it suffices to determine whether $\CD(\H) \leq O(\LD(\H))$ in order to answer the above question.
			We remark that there are classes $\H$ for which $\LD(\H) < \CD(\H)$; for example, the following class $\H\subseteq\{0,1\}^4$ has $\CD(\H)=3$ and $\LD(\H)=2$: 
			\begin{align*}
			    \H=\bigl\{&(\textcolor{red}{0},\textcolor{red}{0},\textcolor{red}{0},1), 
			      (\textcolor{red}{0},\textcolor{red}{1},1,\textcolor{red}{0}),\\
			      &(\textcolor{red}{0},1,\textcolor{red}{1},\textcolor{red}{1}), \tag{The red datasets form a clique.}
			      (1,\textcolor{red}{0},\textcolor{red}{1},\textcolor{red}{0}),\\
			      &(\textcolor{red}{1},\textcolor{red}{0},\textcolor{red}{0},1),
			      (\textcolor{red}{1},\textcolor{red}{1},1,\textcolor{red}{0}),\\
			      &(\textcolor{red}{1},1,\textcolor{red}{1},\textcolor{red}{1}),
			      (1,\textcolor{red}{1},\textcolor{red}{0},\textcolor{red}{1})\bigr\}.
			\end{align*}
			
					
			\paragraph{Sauer-Shelah-Perles Lemma for Fractional Clique Dimension.}
			In \Cref{sec:PDPvsFCD} we proved the analogue to the Sauer-Shelah-Perles Lemma for the fractional clique dimension. 
			We showed a polynomial bound on the fractional clique number of $G_m(\H)$ whenever $\FCD(\H)$ is finite. 
			In contrast to the polynomial-exponential dichotomy satisfied by the clique dimension, where the dimension itself bounds the polynomial degree, in this case
			the obtained bound on the polynomial degree depends on the difference $2^{m}-\omega^\star_{m}$ where $m$ is any number satisfying $2^{m}-\omega^\star_{m}>0$. 		
			\begin{question}
                Let $\H$ be a class with $\FCD(\H)=d<\infty$. Is there a polynomial $P_d(m)$ whose degree depends \emph{only} on $d$ such that $\omega^\star_m(\H)\leq P_d(m)$ for every $m$?
			\end{question}

            \paragraph{Direct Proofs.}
            It will be interesting to find direct proofs for the characterizations of pure and approximate private learnability via the clique and fractional clique dimensions.
            In particular, is there a natural way to construct hard distributions for private learning from large cliques or fractional cliques?

            \paragraph{Expressivity of Contradiction Graphs.} 
              Which learning theoretic properties are definable by the contradiction graph? In this work we demonstrated that online learnability, approximate private learnability, and pure private learnability are captured by the clique and fractional clique numbers of the contradiction graph. How about PAC learnability (which is equivalent to finite VC dimension)? Is the property of having a finite VC dimension detectable in the contradiction graph? More formally, are there two hypothesis classes $\H_1,\H_2$ such that the VC dimension of $\H_1$ is finite and the VC dimension of $\H_2$ is infinite, but $G_m(\H_1)\equiv G_m(\H_2)$ for every $m$? (Here ``$\equiv$'' denote the isomorphism relation between undirected graphs).
					
		\section*{Acknowledgements}
        We thank Ron Holzman, Emanuel Milman and Ramon van Handel for insightful discussions and suggestions surrounding the proof of the strong duality for contradiction graphs (Theorem~\ref{thm:minimax}).
        We also thank Jonathan Shafer for his comments.

			\bibliographystyle{alpha}
			\bibliography{ideas,learning}

\newcommand{\etalchar}[1]{$^{#1}$}
\begin{thebibliography}{BDPSS09}

\bibitem[ABL{\etalchar{+}}22]{AlonBLMM22}
Noga Alon, Mark Bun, Roi Livni, Maryanthe Malliaris, and Shay Moran.
\newblock Private and online learnability are equivalent.
\newblock {\em J. {ACM}}, 69(4):28:1--28:34, 2022.

\bibitem[ALMM19]{ALMM19}
Noga Alon, Roi Livni, Maryanthe Malliaris, and Shay Moran.
\newblock Private {PAC} learning implies finite {L}ittlestone dimension.
\newblock In {\em Proceedings of the 51st Annual ACM SIGACT Symposium on Theory
  of Computing}, STOC 2019, pages 852--860, New York, NY, USA, 2019.
  Association for Computing Machinery.

\bibitem[app16a]{apple1}
Apple promises to deliver {AI} smarts without sacrificing your privacy.
\newblock {\em The Verge}, 2016.

\bibitem[app16b]{apple}
Apple tries to peek at user habits without violating privacy.
\newblock {\em The Wall Street Journal}, 2016.

\bibitem[BDPSS09]{ben2009agnostic}
Shai Ben-David, D{\'a}vid P{\'a}l, and Shai Shalev-Shwartz.
\newblock Agnostic online learning.
\newblock In {\em COLT}, volume~3, page~1, 2009.

\bibitem[BGH{\etalchar{+}}23]{bun2023stability}
Mark Bun, Marco Gaboardi, Max Hopkins, Russell Impagliazzo, Rex Lei, Toniann
  Pitassi, Satchit Sivakumar, and Jessica Sorrell.
\newblock Stability is stable: Connections between replicability, privacy, and
  adaptive generalization, 2023.

\bibitem[BLM20]{BunLM20}
Mark Bun, Roi Livni, and Shay Moran.
\newblock An equivalence between private classification and online prediction.
\newblock In Sandy Irani, editor, {\em 61st {IEEE} Annual Symposium on
  Foundations of Computer Science, {FOCS} 2020, Durham, NC, USA, November
  16-19, 2020}, pages 389--402. {IEEE}, 2020.

\bibitem[BNS13]{beimel2013characterizing}
Amos Beimel, Kobbi Nissim, and Uri Stemmer.
\newblock Characterizing the sample complexity of private learners.
\newblock In {\em ITCS}. ACM, 2013.

\bibitem[BNS19]{Beimel19Pure}
Amos Beimel, Kobbi Nissim, and Uri Stemmer.
\newblock Characterizing the sample complexity of pure private learners.
\newblock {\em Journal of Machine Learning Research}, 20(146):1--33, 2019.

\bibitem[DLS{\etalchar{+}}]{Census20}
Aref~N. Dajani, Amy~D. Lauger, Phyllis~E. Singer, Daniel Kifer, Jerome~P.
  Reiter, Ashwin Machanava-jjhala, Simson~L. Garfinkel, Scot~A. Dahl, Matthew
  Graham, Vishesh Karwa, Hang Kim, Philip Lelerc, Ian~M. Schmutte, William~N.
  Sexton, Lars Vilhuber, and John~M. Abowd.
\newblock {The Modernization of Statistical Disclosure Limitation at the U.S.
  Census Bureau}.
\newblock Presented at the September 2017 meeting of the Census Scientific
  Advisory Committee.

\bibitem[DMNS06]{DMNS06}
Cynthia Dwork, Frank McSherry, Kobbi Nissim, and Adam Smith.
\newblock Calibrating noise to sensitivity in private data analysis.
\newblock In {\em Theory of Cryptography Conference}, pages 265--284. Springer,
  2006.

\bibitem[DR14]{DR14}
Cynthia Dwork and Aaron Roth.
\newblock The algorithmic foundations of differential privacy.
\newblock {\em Foundations and Trends in Theoretical Computer Science},
  9(3-4):211--407, 2014.

\bibitem[EPK14]{Erlingsson14google}
\'{U}lfar Erlingsson, Vasyl Pihur, and Aleksandra Korolova.
\newblock Rappor: Randomized aggregatable privacy-preserving ordinal response.
\newblock In {\em Proceedings of the 2014 ACM SIGSAC Conference on Computer and
  Communications Security}, CCS '14, pages 1054--1067, New York, NY, USA, 2014.
  Association for Computing Machinery.

\bibitem[FX15]{FeldmanX15}
Vitaly Feldman and David Xiao.
\newblock Sample complexity bounds on differentially private learning via
  communication complexity.
\newblock {\em SIAM Journal on Computing}, 44(6):1740--1764, 2015.

\bibitem[ILPS22]{ImpagliazzoLPS22}
Russell Impagliazzo, Rex Lei, Toniann Pitassi, and Jessica Sorrell.
\newblock Reproducibility in learning.
\newblock In Stefano Leonardi and Anupam Gupta, editors, {\em {STOC} '22: 54th
  Annual {ACM} {SIGACT} Symposium on Theory of Computing, Rome, Italy, June 20
  - 24, 2022}, pages 818--831. {ACM}, 2022.

\bibitem[Kak43]{zbMATH03099862}
Shizuo Kakutani.
\newblock Notes on infinite product measure spaces. {I}, {II}.
\newblock {\em Proc. Imp. Acad. Tokyo}, 19:148--151, 184--188, 1943.

\bibitem[Kel75]{kelley1975general}
J.L. Kelley.
\newblock {\em General Topology}.
\newblock Graduate Texts in Mathematics. Springer New York, 1975.

\bibitem[Lit88]{Lit88}
Nick Littlestone.
\newblock Learning quickly when irrelevant attributes abound: A new
  linear-threshold algorithm.
\newblock {\em Machine Learning}, 2:285--318, 1988.

\bibitem[LM20]{LivniM20}
Roi Livni and Shay Moran.
\newblock A limitation of the pac-bayes framework.
\newblock In Hugo Larochelle, Marc'Aurelio Ranzato, Raia Hadsell,
  Maria{-}Florina Balcan, and Hsuan{-}Tien Lin, editors, {\em Advances in
  Neural Information Processing Systems 33: Annual Conference on Neural
  Information Processing Systems 2020, NeurIPS 2020, December 6-12, 2020,
  virtual}, 2020.

\bibitem[MM22]{MalliarisM22}
Maryanthe Malliaris and Shay Moran.
\newblock The unstable formula theorem revisited.
\newblock {\em CoRR}, abs/2212.05050, 2022.

\bibitem[PNG22]{PradeepNG22}
Aditya Pradeep, Ido Nachum, and Michael Gastpar.
\newblock Finite littlestone dimension implies finite information complexity.
\newblock In {\em {IEEE} International Symposium on Information Theory, {ISIT}
  2022, Espoo, Finland, June 26 - July 1, 2022}, pages 3055--3060. {IEEE},
  2022.

\bibitem[Rud87]{rudin1987real}
W.~Rudin.
\newblock {\em Real and Complex Analysis}.
\newblock Mathematics series. McGraw-Hill, 1987.

\bibitem[Rud91]{rudin1991functional}
W.~Rudin.
\newblock {\em Functional Analysis}.
\newblock International series in pure and applied mathematics. McGraw-Hill,
  1991.

\bibitem[Sau72]{sauer1972density}
Norbert Sauer.
\newblock On the density of families of sets.
\newblock {\em Journal of Combinatorial Theory, Series A}, 13(1):145--147,
  1972.

\bibitem[Sio58]{zbMATH03133049}
Maurice Sion.
\newblock On general minimax theorems.
\newblock {\em Pac. J. Math.}, 8:171--176, 1958.

\bibitem[SSBD14]{shalev2014understanding}
Shai Shalev-Shwartz and Shai Ben-David.
\newblock {\em Understanding machine learning: From theory to algorithms}.
\newblock Cambridge university press, 2014.

\bibitem[SU13]{scheinerman13}
Edward~R. Scheinerman and Daniel~H. Ullman.
\newblock {\em Fractional Graph Theory: a Rational Approach to the Theory of
  Graphs}.
\newblock Dover Publications, Minola, N.Y., 2013.

\bibitem[Tao11]{tao2011introduction}
T.~Tao.
\newblock {\em An Introduction to Measure Theory}.
\newblock Graduate studies in mathematics. American Mathematical Society, 2011.

\bibitem[Vad17]{Vadhan17survey}
Salil~P. Vadhan.
\newblock The complexity of differential privacy.
\newblock In {\em Tutorials on the Foundations of Cryptography}, pages
  347--450. Springer International Publishing, 2017.

\bibitem[Val84]{valiant:84}
L.~G. Valiant.
\newblock A theory of the learnable.
\newblock {\em Communications of the ACM}, 27(11):1134--1142, November 1984.

\bibitem[VC68]{vapnik:68}
V.~Vapnik and A.~Chervonenkis.
\newblock On the uniform convergence of relative frequencies of events to their
  probabilities.
\newblock {\em Proc. USSR Acad. Sci.}, 181(4):781--783, 1968.

\end{thebibliography}

			\newpage	
			\appendix
					
			\section{Strong Duality in the Contradiction Graph}\label{app:minimax}
					
			We now turn to prove that for every (possibly infinite) $\H$ and $m\in\mathbb{N}$, the contradiction graph~$G_m(\H)$ satisfies that its fractional chromatic and clique numbers are equal and bounded by $2^m$ (\Cref{thm:minimax}).
			Towards this end it will be convenient to represent fractional colorings using probability distributions over $\{0,1\}^\X$ and fractional cliques using probability distributions over $V_m(\H)$ -- the set of all $m$-datasets that are realizable by $\H$.
					
			Since $\H$ might be infinite, we need to be more careful with specifying which distributions we consider.
			Let us begin with $\{0,1\}^\X$: one of the basic facts we use in the paper is that the fractional chromatic number of $G_m(\H)$ is at most~$2^m$. To prove this, we drew a random hypothesis by sampling its values on each $x\in \X$ uniformly and independently. For this argument to apply to infinite domains $\X$ it is natural to use the product topology on $\{0,1\}^\X$ (i.e.\ the Tychonoff product of the discrete topology over $\{0,1\}$).
			The product structure allows to sample a random hypothesis as above, by taking the product  of the uniform distributions over $\{0,1\}$
			(see \Cref{rmk:well_def_and_reg}).
					
			Formally, we let the set of fractional colorings to be the space of all Borel\footnote{A Borel measure on a topological space $X$ is a measure defined on the Borel $\sigma$-algebra, which is the $\sigma$-algebra generated by all open sets in $X$.} 
			regular\footnote{Roughly speaking, a regular measure is a measure that can be approximated from above by open sets and from below by compact sets. See definition in \Cref{sec:perliminaries_top}.} 
			probability measures over $\{0,1\}^\X$, where the latter is equipped with the product topology (Tychonoff). Denote this space by~$\Delta\bigl(\{0,1\}^\X\bigr)$.
			(The relevant concepts in topology are defined below.)
					
			We now turn to define fractional cliques, which are distributions over the set of realizable datasets of size $m$, $V_m(\H)$. Here, we simply consider finitely supported distributions and use the trivial discrete topology and the corresponding Borel $\sigma$-algebra over $V_m(\H)$, both equal to the entire powerset. Denote the space of all finitely supported probability distributions over $V_m(\H)$ by~$\Delta(V_m^\H)$.
					
			We next extend the definitions of fractional clique and chromatic numbers using the above structures. The definitions rely on the following basic claim. 
			\begin{claim}
				The function $f:V_m(\H)\times \{0,1\}^\X\to\mathbb{R}$ defined by 
				\[f(S,h)=\Ind[\text{$h$ is consistent with $S$}],\] 
				is continuous and hence Borel-measurable with respect to the product topology over $V_m(\H)\times \{0,1\}^\X$.
			\end{claim}
			\begin{proof}
				Because open sets are union-closed, it is enough to show that for each fixed $S\in V_m(\H)$, the set $\{(S,h): \text{$h$ is consistent with $S$}\}$ is open. Since $\{S\}$ is open in the discrete topology on $V_m(\H)$, it is enough to show that ${\{h: \text{$h$ is consistent with $S$}\}}$ is open in the product topology on $\{0,1\}^\X$. The latter is indeed true because this set is a basic open set in the product topology (see below for definitions and topological background).
			\end{proof}

			\begin{definition}[Fractional Chromatic and Clique Numbers]
				Let $\H\subseteq \{0,1\}^\X$ be a class and $m\in \Nat$. 
				The \emph{fractional chromatic number} of $G_m(\H)$, denoted $\chi^\star_m$, is defined by
				\[\frac{1}{\chi^\star_m}=\sup_{\mu\in \Delta\bigl(\{0,1\}^\X\bigr)}\inf_{\nu\in \Delta(V_m^\H)}\expect_{\substack{h\sim\mu,\\S\sim\nu}}\bigl[\Ind[\text{$h$ is consistent with $S$}]\bigr],\]
				with the convention that $\frac{1}{0}=\infty$ and $\frac{1}{\infty}=0$.
							    
				The \emph{fractional clique number} of $G_m(\H)$, denoted $\omega^\star_m$, is defined by
				\[\frac{1}{\omega^\star_m}=\inf_{\nu\in \Delta(V_m^\H)}\sup_{\mu\in \Delta\bigl(\{0,1\}^\X\bigr)}\expect_{\substack{h\sim\mu,\\S\sim\nu}}\bigl[\Ind[\text{$h$ is consistent with $S$}]\bigr],\]
				with the convention that $\frac{1}{0}=\infty$ and $\frac{1}{\infty}=0$.
			\end{definition}

			\begin{remark}\label{r:24}
				The definition below of the fractional chromatic and clique numbers involves taking the expectation of the random variable $\Ind[\text{$h$ is consistent with $S$}]$ when $h\sim \mu\in \Delta\bigl(\{0,1\}^\X\bigr) $ and $S\sim \nu\in \Delta(V_m^\H)$.
				Note that because the distribution $\nu$ is finitely supported, this expectation can be expressed as a finite sum:
				\[\expect_{\substack{h\sim\mu,\\S\sim\nu}}\bigl[\Ind[\text{$h$ is consistent with $S$}]\bigr] =\sum \nu(S)\cdot \expect_{h\sim \mu}[\Ind[\text{$h$ is consistent with $S$}]],\]
				where the sum ranges over all (finitely many) datasets $S$ in the support of $\nu$.
			\end{remark}
					

			\begin{remark}\label{rmk:well_def_and_reg}
				In \Cref{sec:well_def_and_reg} we show that there exists a Borel-regular probability measure over $h\in \{0,1\}^\X$ which corresponds to sampling $h(x)\in\{0,1\}$ uniformly and independently for each $x\in\X$. Thus, the corresponding fractional coloring which witnesses that $\chi^\star_m \leq 2^m$ is well-defined.
							
			\end{remark}

			\paragraph{Sion's Theorem.} 
			Our proof of \Cref{thm:minimax} relies on Sion's theorem, which is a generalization of Von-Neumann's minimax theorem. 
			Note that we state a slightly weaker version than the original Sion's theorem \cite{zbMATH03133049}. 
					
			\begin{theorem}[Sion's Theorem -- weak version]
				Let $W$ be a compact convex subset of a linear topological space\footnote{A linear topological space is a vector space that is also a topological space with the property that the vector space operations (vector addition and scalar multiplication) are continuous.} and $U$ a convex subset of a linear topological space. If $F$ is a real-valued function on $W\times U$ such that
				\begin{enumerate}
					\item $F(w,\cdot)$ is linear and continuous on $U$ for every $w\in W$, and
					\item $F(\cdot, u)$ is linear and continuous on $W$ for every $u\in U$
				\end{enumerate}
				then,
				\[\max_{w\in W} \inf_{u\in U} F(w,u)=\inf_{u\in U} \max_{w\in W} F(w,u). \]
			\end{theorem}
					
			In order to apply Sion's Theorem we start by defining appropriate topologies on $\Delta\bigl(\{0,1\}^\X\bigr)$ and $\Delta(V_m^\H)$.
			We start by presenting useful facts from topology and functional analysis. Readers who are familiar with this material may skip it and continue to \Cref{sec:proofminimax}
					
			\subsection{Preliminaries from Topology and Analysis}\label{sec:perliminaries_top}
					
			\paragraph{Product Topology.} 
			Let $\{0,1\}^\X$ be the space of all functions ${f:\X\to\{0,1\}}$. The product topology on $\{0,1\}^\X$ is the coarsest topology in which each projection is continuous; i.e, for every $x\in \X$, the mapping $\pi_x:\{0,1\}^\X\to\{0,1\}$ defined by $\pi_x(f)=f(x)$, is continuous.
			A basis of open sets of the product topology is given by sets of the form
			\[U_{(x_1,y_1)\ldots,(x_m,y_m)}=\{f\mid f(x_i)=y_i \text{ for all } 1\leq i\leq m \},\]
			where $m\in \Nat$ and $\bigl((x_1,y_1),\ldots,(x_m,y_m)\bigr)\in{\bigl(\X\times\{0,1\}\bigr)^m}$.
			Note that every such basic open set is also closed since its complement is a union of $2^m-1$ basic sets (such sets are often called \emph{clopen} because they are both closed and open).
			By Tychonoff's theorem (e.g.\ see Theorem 5.13 in \cite{kelley1975general}), the space $\{0,1\}^\X$ is compact for any set $\X$.
					
			\paragraph{Total Variation.}
			We next state some basic useful facts on total variation. For more information see Chapter 6 at \cite{rudin1987real}.    
			Let $(X,\Sigma, \mu)$ be a measure space where $\mu$ is a signed measure\footnote{A signed measure is a $\sigma$-additive set-function that can assign negative values.}.
			The \emph{total variation measure} of $\mu$ is a positive measure on $(X,\Sigma)$ defined as follows
			\[|\mu|(E)=\sup\sum_{i=1}^{\infty}|\mu(E_i)| , \quad E\in\Sigma\]
			where the supremum ranges over all countable partitions $\{E_i\}_{i=1}^{\infty}$ of $E$.
			Note that for every measurable set $E$, $|\mu|(E)\geq|\mu(E)|$.
			If $|\mu|$ is finite, we say that $\mu$ has \emph{bounded variation}, and the \emph{total variation} of $\mu$ is defined to be
			\[\Vert \mu\Vert_{TV}=|\mu|(X).\]
			The set of all signed measures with bounded variation, together with $\Vert\cdotp\Vert_{TV}$, is a normed linear space.
			Given a signed measure $\mu$,
			define
			\[\mu^+=\frac{1}{2}(|\mu|+\mu),\quad \mu^-=\frac{1}{2}(|\mu|-\mu).\]
			Then, both $\mu^+$ and $\mu^-$ are positive measures on $(X,\Sigma)$ which are called the \emph{positive} and \emph{negative variations} of $\mu$. Also we have that $\mu= \mu^+-\mu^-, |\mu|= \mu^++\mu^-$,
			and therefore
			\[\Vert\mu\Vert_{TV}=|\mu|(X)=\mu^+(X)+\mu^-(X)= \Vert\mu^+\Vert_{TV}+\Vert\mu^-\Vert_{TV}.\]
			This representation of $\mu$ as the difference of the positive measures $\mu^+$ and $\mu^-$ is called the \emph{Jordan Decomposition} of $\mu$. 
					
			\paragraph{Regularity.} A measure $\mu$, defined on a $\sigma-$algebra $\Sigma$ of a Hausdorff space $X$, is called \emph{inner regular} if for every $E\in\Sigma$, 
			\[\mu(E)= \sup_{K\in\Sigma}(\mu(K)\mid K\subseteq E \text{ is compact}).\]
			The measure $\mu$ is called \emph{outer regular} if for every $E\in\Sigma$, 
			\[\mu(E)= \inf_{O\in\Sigma}(\mu(O)\mid O\supset E \text{ is open}).\]
			The measure $\mu$ is \emph{regular} if it is both inner regular and outer regular. Note that if $X$ is compact and $\mu$ is finite, inner regularity is equivalent to outer regularity.

			\paragraph{Weak$^\star$ Topology and the Dual of $C(K)$.} 
			Recall that for a linear space $V$, the space $V^\star$ denotes the dual space, i.e.\ the space of all linear functionals on~$V$. 
			Let $X$ be a linear topological space over $\mathbb{R}$. For $x\in X$, let $T_x\in X^{\star\star}$ denote the evaluation operator: $T_x(f)=f(x)$, where $f\in X^\star$ is a linear functional  $f:X\to \R$. 
			The \emph{weak$^\star$ topology} on $X^\star$ is the coarsest topology such that the operators $T_x$ are continuous.
					
			Let $K$ be a compact Hausdorff space. Denote by $C(K)$ the set of all real-valued continuous functions on $K$, 
			and denote by $\mathcal{B}(K)$ the set of all finite signed regular Borel measures on $K$. We treat both $C(K)$ and $\mathcal{B}(K)$ as linear normed spaces, the first is equipped with max-norm and the second with total-variation norm.
			By Riesz-Markov Representation Theorem $C(K)^\star$ and $\mathcal{B}(K)$ are isometric:
			\begin{theorem}[Riesz, Theorem 6.19, \cite{rudin1987real}]
				Let $K$ be a 
				compact Hausdorff space. Then, every bounded linear functional $\phi$ on $C(K)$ 
				is represented by a unique regular finite 
				Borel signed\footnote{Theorem 6.19 in \cite{rudin1987real} considers complex functionals and complex measures. However, one can show that if the functional is real then the corresponding regular measure must also be real. One way to see it is to use Urysohn's Lemma to show that a regular complex measure for which all continuous functions have real integrals is in fact a real measure. }
				measure~$\mu$, in the sense that 
				\[\phi f =\int_K f \,d\mu\]
				for every $f\in C(K)$.
				Moreover, the operator norm of $\phi$ is the total variation of~$\mu$:
				\[\Vert \phi\Vert= \Vert \mu\Vert_{TV}.\]
			\end{theorem}
			This natural identification between $C(K)^\star$ and $\mathcal{B}(K)$ allows us to define a topology on $\mathcal{B}(K)$ in terms of a topology on $C(K)^\star$. 
			We do so by considering the  weak$^\star$ on $C(K)^\star$. 
			Thus, this is the coarsest topology such that the operators
			\[T_f(\mu)=\int f \,d\mu\]
			are continuous for every $f\in C(K)$.
			We now use Banach-Alaoglu Theorem (Theorem~3.15 in~\cite{rudin1991functional}) which implies that the closed unit ball of $C(K)^\star$ is closed with respect to the weak$^\star$ topology. Consequently by the above identification of $C(K)^\star$ and $\mathcal{B}(K)$ we obtain:
			\begin{claim}\label{clm:banach}
				Let $K$ be a compact Hausdorff space. Then the closed unit ball of~$\mathcal{B}(K)$,
				\[B\bigl[\mathcal{B}(K)\bigr]:=\{\mu\in\mathcal{B}(K)\mid \Vert \mu\Vert_{TV}\leq1\},\] 
				is compact in the weak$^\star$ topology.
			\end{claim}
			
			\subsection{Proof of Theorem~\ref{thm:minimax} [Strong duality in the contradiction graph]}\label{sec:proofminimax}
					
			\begin{theorem*}[\Cref{thm:minimax} Restatement]
				Let $\X$ be an arbitrary domain, $\H\subseteq\{0,1\}^\X$ a concept class and $m\in\mathbb{N}$.
            	Let $\omega_m^\star$ and $\chi_m^\star$ denote the fractional clique and chromatic numbers of the contradiction graph $G_m(\H)$.
            	Then, 
            	\[\omega_m^\star = \chi_m^\star\leq 2^m.\] 
            	Moreover, there exists a fractional coloring realizing $\chi_m^\star$.
            	(I.e.\ the infimum is in fact a minimum.)
			\end{theorem*}
					
			\begin{proof}[Proof of \Cref{thm:minimax}]
				In order to prove the equality $\omega_m^\star = \chi_m^\star$
				we use Sion's Theorem.
				Recall the extended definitions of fractional colorings and cliques of $G_m(\H)$: fractional colorings are regular Borel probability measures on $\{0,1\}^X$ (where the latter is equipped with the product topology). The set of fractional colorings is denoted $\Delta(\{0,1\}^X)$.
				Fractional cliques are finitely supported distributions over $V_m(\H)$, the set of $\H$-realizable datasets of size $m$.
				The set of fractional cliques is denoted $\Delta(V_m^\H)$.
							
				We plug in Sion's theorem  $W=\Delta\bigl(\{0,1\}^\X\bigr)$, $U= \Delta(V_m^\H)$, and
				\[F(\mu,\nu)=\expect_{\substack{h\sim \mu \\ S\sim \nu}}[\Ind[\text{$h$ is consistent with $S$}]].\]
				Note that by definition of $F$,
				\begin{align*}
					  & \sup_{w\in W} \inf_{u\in U} F(w,u)=\frac{1}{\chi_m^\star},   \\
					  & \inf_{u\in U} \sup_{w\in W} F(w,u)=\frac{1}{\omega_m^\star}. 
				\end{align*}
				Applying Sion's Theorem with $W$,$U$ and $F$ as above proves the equality $\omega_m^\star = \chi_m^\star$, and the claim that there exists a fractional coloring realizing $\chi_m^\star$.
				In order to apply Sion's theorem on $W$, $U$, $F$, we need to define the corresponding linear topological spaces
				and verify that the assumptions in the premise of Sion's theorem are satisfied.
							

				\paragraph{Step 1: $\Delta\bigl(\{0,1\}^\X\bigr)$ is a compact convex subset of a linear topological space.}
				$\Delta\bigl(\{0,1\}^\X\bigr)$ is a subset of $\mathcal{B}(\{0,1\}^\X)$, which is the set of all signed finite regular Borel measures on $\{0,1\}^\X$. We consider $\mathcal{B}(\{0,1\}^\X)$ as a linear topological space equipped with the weak$^\star$ topology.
				Clearly, $\Delta(\{0,1\}^X)$ is convex, and thus it remains to show that it is compact.
				Notice that $\Delta(\{0,1\}^X)$ is contained in the unit ball $B[\mathcal{B}(\{0,1\}^X)]$, and the latter is compact by \Cref{clm:banach}. Thus, it remains to show that $\Delta(\{0,1\}^X)$ is closed in the weak$^\star$ topology, 
				which follows from the next claim.
							
				\begin{claim}\label{clm:W_closed}
					Let
					\begin{align*}
						B & =B\bigl[\mathcal{B}\bigl(\{0,1\}^\X\bigr)\bigr]= \Bigl\{\mu\in\mathcal{B}\bigl(\{0,1\}^\X\bigr)\mid \Vert \mu\Vert_{TV}\leq1\Bigr\}, \\
						S & =\Bigl\{\mu\in\mathcal{B}\bigl(\{0,1\}^X\bigr)\mid \mu\bigl(\{0,1\}^\X\bigr)=1\Bigr\}.                                               
					\end{align*}
					Then, $\Delta\bigl(\{0,1\}^\X\bigr)=B\cap S$.
				\end{claim}
							
				Note that the above claim implies that $\Delta\bigl(\{0,1\}^\X\bigr)$ is weak$^\star$-closed as an intersection of two closed subsets:
				$B$ is weak$^\star$-compact (and in particular closed) by \Cref{clm:banach}. 
				To see that $S$ is closed, consider the operator $T_1(\mu)=\int 1 \,d\mu=\mu(\{0,1\}^X)$, 
				which is weak$^\star$-continuous because the constant map $1$ is continuous. 
				Thus, $S=T_1^{-1}(\{1\})$ is closed.
							
				\begin{proof}[Proof of \Cref{clm:W_closed}]
					It is clear that $\Delta\bigl(\{0,1\}^\X\bigr)$ is contained in the right-hand side. From the other direction, it is enough to show that if $\mu\in B\cap S$ then~$\mu\geq0$. Write $\mu=\mu^+-\mu^-$, where $\mu^+,\mu^-$ are the positive and negative variations of $\mu$.
					Since $\mu\in B$, 
					\begin{equation}\label{eq:sum_total_var_leq1}
						\Vert\mu\Vert_{TV}=\Vert\mu^+\Vert_{TV}+\Vert\mu^-\Vert_{TV}\leq1.  
					\end{equation}
					Since $\mu\in S$,
					\begin{align}
						1=\mu(\{0,1\}^X) & =\mu^+(\{0,1\}^X)-\mu^-(\{0,1\}^X) \nonumber                                                       \\
						                 & =|\mu^+|(\{0,1\}^X)-|\mu^-|(\{0,1\}^X) \nonumber\tag{because $\mu^+, \mu^-$ are positive measures} \\
						                 & = \Vert\mu^+\Vert_{TV}-\Vert\mu^-\Vert_{TV}. \label{eq:diff_total_var}                             
					\end{align}
					From \Cref{eq:sum_total_var_leq1,eq:diff_total_var} we conclude $\Vert\mu^-\Vert_{TV}=0$, i.e. $\mu^-=0$, and $\mu=\mu^+$ is a positive, completing the proof.
				\end{proof}
							
				\paragraph{Step 2: $\Delta(V_m^\H)$ is a convex subset of a linear topological space.}
				Consider the space of all finitely supported signed measures on $V_m(\H)$ with the discrete topology.
				$\Delta(V_m^\H)$ is a convex subset of this space.
							
				\paragraph{Step 3: $F$ is linear and continuous in each coordinate.}
				First, notice that $F(\mu,\nu)$ is linear in each coordinate by linearity of expectation with respect to the underlying measure.
							
				We next show that $F(\cdot, \nu)$ is continuous for every fixed $\nu$.
				By \Cref{r:24}, for every fixed $\nu$, $F(\cdot, \nu)$ is a finite convex combination
				of functions of the form 
				\[F_S(\mu) = \expect_{\substack{h\sim \mu}}[\Ind[\text{$h$ is consistent with $S$}]],\]
				where $S$ is a fixed $\H$-realizable dataset of size $m$.
				Thus it suffices to show that $F_S(\mu)$ is weak$^\star$-continuous for each fixed $S$.
				Let $f_S(h) = \Ind[\text{$h$ is consistent with $S$}]$ and notice that $F_S(\mu) = \expect_{h\sim \mu}[f_S(h)]$.
				Thus, by the definition of the weak$^\star$ topology, it is enough to show that
				the map $f_S$ is continuous with respect to the product topology on $\{0,1\}^\X$.
				Since $f_S$ is an indicator map (with values in $\{0,1\}$),
				continuity amounts to showing that both $f_S^{-1}(0)$ and $f_S^{-1}(1)$ are open sets.
				Indeed, by definition $f_S^{-1}(1)$ is a basic open set which is also closed and hence $f_S^{-1}(0)$ is also open as required.	
				It remains to show that $F(\mu, \cdot)$ is continuous in its second coordinate for every fixed~$\mu$.
				This is straightforward because the topology on the second coordinate ($\Delta(V_m^\H)$) is discrete and hence every function is continuous.
				
				\medskip
				            
				To complete the proof of \Cref{thm:minimax}, it is left to show that $\chi^m\leq2^m$.
				Recall, 
				\[\frac{1}{\chi^\star_m}=\sup_{\mu\in \Delta\bigl(\{0,1\}^\X\bigr)}\inf_{\nu\in \Delta(V_m^\H)}\expect_{\substack{h\sim\mu,\\S\sim\nu}}\bigl[\Ind[\text{$h$ is consistent with $S$}]\bigr].\]
				Hence, it suffices to show that there exists a distribution $\mu^\star\in \Delta\bigl(\{0,1\}^\X\bigr)$ such that for every realizable dataset $S=\bigl((x_1,y_1),\ldots,(x_m,y_m)\bigr)\in V_m(\H)$,
				\[\Pr_{h\sim\mu^\star}[h\text{ is consistent with }S]=\mu^\star\bigl(\{h\mid h(x_i)=y_i, i=1,\ldots,m\}\bigr)=\frac{1}{2^m}.\]
				Indeed if this holds, then by definition $\frac{1}{\chi^\star_m}\geq\frac{1}{2^m}$.
				Therefore, the following lemma concludes the proof.
				            
				\begin{lemma}\label{l:well_def_and_reg}
					Let $\X$ be an arbitrary domain. Then there exists a regular Borel probability measure~$\mu^\star$ over $\{0,1\}^\X$, which satisfies the following property. For every 
					finite set $\{x_1,\ldots,x_k\}\subseteq \X$ and labels $y_1,\ldots, y_k$
					\[\mu^\star\bigl(\{h\in \{0,1\}^\X\mid h(x_i)=y_i, i=1,\ldots ,k\}\bigr)=\frac{1}{2^k}.\]
				\end{lemma}
				
				Before turning to prove \Cref{l:well_def_and_reg}, we first need to recollect some definitions and state known results from measure theory which are needed for the proof. The proof is deferred to \Cref{sec:well_def_and_reg}.

							
							
			\end{proof}

			\subsection{Tossing a fair coin infinitely many times}\label{sec:well_def_and_reg}
			One of the basic facts we used in the paper is that the fractional chromatic number of the contradiction graph is bounded. To prove this fact we drew a random hypothesis by sampling a value independently for each $x\in \X$ uniformly over $\{0,1\}$.
					
			For a countable domain $\X$ it is clear that this sampling process induces a well-defined probability distribution over $\{0,1\}^\X$. However, if $\X$ is arbitrary, and possibly uncountable, it is not clear that this process is even well-defined. 
					
			In this section, we prove \Cref{l:well_def_and_reg}: we show that indeed the described sampling process induces a probability measure on the product space $\{0,1\}^\X$. Furthermore, we show that this measure is a regular Borel measure, and hence a valid fractional coloring. (Recall that we identify fractional colorings as regular Borel probability measures over $\{0,1\}^\X$, where the latter is equipped with the product topology).

			\subsubsection{Measure Theory Preliminaries}\label{sec:measure_perliminaries}
					
			\paragraph{The Product $\sigma$-Algebra.} 
			Let $(X_\alpha, \Sigma_\alpha)_{\alpha\in \I}$ be a family of measurable spaces. Denote by ${X_\I=\prod_{\alpha\in \I}X_\alpha}$ the product space, and by $\pi_\alpha:X_\I\to X_\alpha$ the projection map to $X_\alpha$. 
			Using the projection $\pi_\alpha$, we can pull back the $\sigma$-algebra $\Sigma_\alpha$ on  $X_\alpha$, to a $\sigma$-algebra on~$X_\I$:
			\[\pi_\alpha^*(\Sigma_\alpha)\coloneqq\{\pi_\alpha^{-1}(E_\alpha)\mid E_\alpha\in \Sigma_\alpha\}.\]
			Define the \emph{product $\sigma$-algebra} on $X_\I$, denoted $\Sigma_\I$, to be the $\sigma$-algebra generated by all $\pi_\alpha^*(\Sigma_\alpha)$:
			\[\Sigma_\I\coloneqq\Bigl\langle\bigcup_{\alpha\in \I}\pi_\alpha^*(\Sigma_\alpha)\Bigr\rangle=\langle\pi^{-1}_\alpha(E_\alpha)\mid\alpha\in\I\rangle.\]
			Using similar notations, for every subset of indices $A\subseteq \I$, let ${X_A=\prod_{\alpha\in A}X_\alpha}$, let $\pi_A:X_\I\to X_A$ denote the projection to $X_A$, and denote the product $\sigma$-algebra on $A$ by \[\Sigma_A=\Bigl\langle\bigcup_{\alpha\in A}\pi_A\circ \pi_\alpha^*(\Sigma_\alpha)\Bigr\rangle.\] Given $A\subseteq \I$, a measure $\mu_\I$ on the product $\sigma$-algebra
			$\Sigma_\I$, induces a measure $\mu_A$ on $\Sigma_A$ (which is the pushforward of $\mu_\I$):
			\[\mu_A(E)\coloneqq(\pi_A)_*\mu_\I(E)=\mu_\I(\pi_A^{-1}(E)).\]
			Note that those measures obey the compatibility relation:
			For $ B\subseteq A\subseteq I$, denote by $\pi_{A\to B}$ the projection $\restr{\pi_B}{X_A}:X_A\to X_B$.
			Then 
			\begin{equation}\label{eq:compatibility_of_measures}
				\mu_B=(\pi_{A\to B})_*\mu_A ,
			\end{equation}
			i.e, for every $E\in \Sigma_B$, $\mu_B(E)=\mu_A((\pi_{A\to B})^{-1}(E))$.

			\paragraph{Kolmogorov's Extension Theorem.}
			A natural question is whether one can reconstruct $\mu_\I$ solely from the projections $\mu_A$ to \emph{finite} subsets $A\subseteq\I$. It turns out that in the special case where the $\mu_A$ are probability measures (and satisfy some additional regularity and compatibility conditions), it is indeed possible. This is the content of Kolmogorov's Theorem (Theorem 2.4.3 in \cite{tao2011introduction}). 
					
			\begin{theorem}[Kolmogorov's Extension Theorem]\label{thm:kolmogorov}
				Let $(X_\alpha, \Sigma_\alpha,\tau_\alpha)_{\alpha\in \I}$ be a family of measurable spaces $(X_\alpha, \Sigma_\alpha)$, equipped with a topology $\tau_\alpha$.
				For every finite subset $A\subseteq\I$, let $\mu_A$ be a probability measure on $\Sigma_A$ which is inner regular with respect to the product topology $\tau_A\coloneqq\prod_{\alpha\in A}\tau_\alpha$ on $X_A$. 
				Furthermore, assume $\mu_A, \mu_B$ obey the compatibility condition in \Cref{eq:compatibility_of_measures} for every
				nested finite subsets $B\subseteq A$ of $\I$.
				Then, there exists a unique probability measure $\mu_\I$ on the product $\sigma$-algebra $\Sigma_\I$, with the property that $(\pi_A)_*\mu_\I=\mu_A$ for all finite subset~$A\subseteq \I$. (I.e.\ the push-forward measure of $\mu_\I$ when projected on $A$ equals to $\mu_A$.)
			\end{theorem}
					
			\paragraph{Extending a product measure to a regular Borel measure.}
			Let us assume here and below that each $X_\alpha$ is a compact Hausdorff space, and $\Sigma_\alpha$ is the Borel
			$\sigma$-algebra on~$X_\alpha$.
			Kolmogorov's theorem allows to construct a probability measure on the product $\sigma$-algebra by specifying its behavior over finite projections. 
			The following theorem allows us to further extend this measure to the Borel $\sigma$-algebra of the product space $X_\I$. (Notice that the Borel $\sigma$-algebra is finer than the Product $\sigma$-algebra.)
					
			\begin{theorem}[Theorem 2, \cite{zbMATH03099862}]\label{thm:extension_of_baire}
				Let $(X_\alpha,\Sigma_\alpha)_{\alpha\in \I}$ be a family measurable spaces where each~$X_\alpha$ is a compact Hausdorff space and $\Sigma_\alpha$ is the Borel $\sigma$-algebra of $X_\alpha$. Then, every probability measure on the product $\sigma$-algebra, $\Sigma_\I$, has a unique extension to a regular Borel measure $\mu^*_\I$ on~$X_\I$.
			\end{theorem}
					
			\subsubsection{Proof of Lemma~\ref{l:well_def_and_reg}}
					
			\begin{proof}
				For every finite set of unlabeled examples $\{x_1,\ldots,x_k\}\subseteq \X$, let $\mu_{\{x_1,\ldots,x_k\}}$ be the uniform distribution measure over $\{0,1\}^{\{x_1,\ldots,x_k\}}$. Notice that $\{\mu_{\{x_1,\ldots,x_k\}}: \{x_1,\ldots,x_k\}\subseteq \X\}$ are regular and obey the the compatibility condition in \Cref{eq:compatibility_of_measures}. By \Cref{thm:kolmogorov} there exists a probability measure $\mu$ over the product $\sigma$-algebra of $\{0,1\}^\X$ which satisfies the equations stated in the Lemma.
				By \Cref{thm:extension_of_baire} there exists a regular Borel probability measure $\mu^\star$ as required.
			\end{proof}

			\section{Additional Proofs}\label{app:proofs}
					
			\subsection{Proof of Lemma~\ref{l:colorings-algs} [Independent sets and consistent hypotheses]}

   \begin{lemma*}[Lemma~\ref{l:colorings-algs}, Restatement]
		Let $\H$ be a class and $m$ be a natural number.
		\begin{enumerate}
			\item For every independent set $I$ in $G_m(\H)$, there exists an hypothesis ${h\in \{0,1\}^\X}$ such that $h$ is consistent with every dataset $S\in I$; i.e.\ for every dataset $S\in I$ and every example $(x,y)\in S$, we have $h(x)=y$.
			\item For every hypothesis $h$, the set of all datasets of size $m$ that are consistent with $h$ is an independent set in $G_m(\H)$; i.e.\ the set
			      \[V_{h}\coloneqq\{S=((x_1,y_1),\ldots, (x_m,y_m))\in V_m(\H)\mid \forall i, h(x_i)=y_i\}\]
			      is independent in $G_m(\H)$
		\end{enumerate}
	\end{lemma*}
					
			\begin{proof}
				First, we prove the first part of the lemma. 
				Let $I$ be an independent set in $G_m(\H)$ and denote 
				\[L(I)=\bigl\{(x,y)\in \X\times\{0,1\}\mid \text{$\exists S\in I$ s.t.\ $(x,y)\in S$}\bigr\}.\]
				Observe that since $I$ is independent, if $(x,y)\in L(I)$ then $(x,1-y)$ must not be in $L(I)$ (because datasets containing $(x,1-y)$ are connected with an edge to datasets containing $(x,y)$).
				Define an hypothesis $h$ as follows: for every $(x,y)\in L(I)$ set $h(x)=y$. For every $x$ such that neither $(x,0)$ or $(x,1)$ are in $L(I)$, set arbitrarily $h(x)=0$. Indeed $h$ interpolates every dataset $S\in I$.
							
				Next, we prove the second part of the lemma.
				Let $h\in\{0,1\}^\X$ be an hypothesis. $V_{h}$ is indeed independent in $G_m(\H)$, else there must be $S,S'\in V_{h}$ and an unlabeled example $x\in \X$ such that $(x,0)\in S, (x,1)\in S'$, and by definition $h(x)=0=1$, leading to a contradiction.
			\end{proof}
					
			\subsection{Proof of Lemma~\ref{l:colorings_cliques_dists} [Fractional cliques and colorings vs. distributions]}

\begin{lemma*}[Lemma~\ref{l:colorings_cliques_dists}, Restatement]
		Let $\H$ be a class, $m\in\Nat$. Then,
		\begin{enumerate}
			\item There exists a fractional coloring $c$ of $G_m(\H)$ with $\col(c)=\alpha>0$ if and only if there exists a distribution $\mu$ over hypotheses such that
			      \[\inf_{S}\Pr_{h\sim\mu}\left[\text{$h$ is consistent with $S$}\right]=\frac{1}{\alpha},\]
			      where the infimum is taken over realizable datasets of size $m$. \label{itm:colorings_dists}
			\item There exists a fractional clique $\delta$ of $G_m(\H)$ with $\lvert\delta\rvert=\alpha>0$ if and only if there exists a distribution $\nu$ over realizable datasets of size $m$ such that
			      \[\sup_{h}\Pr_{S\sim\nu}\left[\text{$h$ is consistent with $S$}\right]=\frac{1}{\alpha},\]
			      where the supremum is taken over hypotheses $h\in \{0,1\}^\X$.\label{itm:cliques_dists}
		\end{enumerate}
	\end{lemma*}
   
			\begin{proof}
				As demonstrated in \Cref{sec:preliminaries_graph_theory}, there is a correspondence between fractional colorings and independent sets: a fractional coloring of $G_m(\H)$ with $\alpha$ colors corresponds to a distribution $\mu$ over independent sets of $G_m(\H)$  with 
				\[\val(\mu)=\frac{1}{\alpha}=\inf_{S \in V_m(\H)}{\Pr_{I \sim \mu}\left[S \in I\right]}.\]
				From \Cref{l:colorings-algs}, there is a correspondence between (maximal) independent sets and hypotheses: an independent set $I$ in $G_m(\H)$ corresponds to an hypothesis $h$ which is consistent with all datasets $S\in I$. 
				Note that a realizable dataset $S$ of size $m$ belongs to an independent set $I$ if and only if  $S$ is consistent with the hypothesis $h$ corresponds to $I$.
				Hence, by abuse of notation, consider $\mu$ to be a distribution over hypotheses. 
				Observe that for every realizable dataset $S$ of size $m$,
				\[\Pr_{h\sim\mu}\left[\text{$h$ is consistent with $S$}\right]=\Pr_{I\sim\mu}\left[S \in I\right].\]
				This completes \Cref{itm:colorings_dists} in the lemma.
				Note that \Cref{itm:cliques_dists} is immediate, again it follows from the discussion in \Cref{sec:preliminaries_graph_theory}, since the vertices of the contradiction are realaizable datasets of size~$m$.
			\end{proof}

            \subsection{Proof of Lemma~\ref{l:CD_bound}}

 \begin{lemma*}[Lemma~\ref{l:CD_bound}, Restatement]
		Let $\H$ be a hypothesis class. Then
		\[\CD(\H)\leq \max\bigl\{2\LD(\H) \log(\LD(\H)), 300\bigr\}.\]
	\end{lemma*}

            In order to show the desired bound we will use the following technical lemma.
	\begin{lemma}\label{l:tech_cd}
		Let $d\in \mathbb{N}$ and $m_0\in\mathbb{R}$ such that
		\begin{itemize}
			\item[(i)] $m_0\geq \frac{d}{\ln 2}$, and
			\item[(ii)] $2^{m_0} \geq (2m_0+1)^d$.
		\end{itemize}
		Then, 
		\[(\forall m > m_0): 2^m > (2m+1)^d.\]
		Further, $m_0=2d\log d$ satisfies items (i) and (ii) above, provided that $d\geq 30$.
	\end{lemma}
	
	\begin{proof}[Proof of \Cref{l:CD_bound}]
		If $\LD(\H)=\infty$ then the inequality trivially holds. Suppose $\LD(\H)=d<\infty$. We distinguish between two cases:
		if $d<30$ then by \Cref{l:clique_number_bounded}, 
		\[\omega_m\leq (2m+1)^d<(2m+1)^{30}\]
		which is less than $2^m$ for every $m\geq300$, and thus $\CD(\H)\leq 300\leq \max\{2d \log d, 300\}$ as required.
		Else, if $d\geq30$ then by \Cref{l:tech_cd} for every $m>2d\log d$
		\[\omega_m\leq (2m+1)^d < 2^m,\]
		thus $\CD(\H) \leq 2d\log(d) \leq \max\{2d \log d, 300\}$ as required.
	\end{proof}

			\begin{proof}[Proof of Lemma~\ref{l:tech_cd}]
				We first show that $2^m>(2m+1)^d$ for all $m>m_0$.
				Indeed, if $m>m_0$ then 
				\begin{align*}
					2^m & = 2^{m-m_0}\cdot 2^{m_0}                                  \\
					    & \geq  2^{m-m_0}\cdot (2m_0+1)^d. \tag{by assumption (ii)} 
				\end{align*}
				Note that
				\begin{align*}
					\Bigl(\frac{2m+1}{2m_0+1}\Bigr)^d & = \Bigl(1+\frac{2(m-m_0)}{2m_0+1}\Bigr)^d                                                 \\
					                                  & < \exp\Bigl(\frac{2(m-m_0)\cdot d}{2m_0+1}\Bigr) \tag{$1+x < \exp(x)$ for all $x>0$}      \\
					                                  & \leq \exp\Bigl(\ln(2)(m-m_0)\Bigr) \tag{$\frac{2d}{2m_0+1} \leq \ln 2$ by assumption (i)} \\
					                                  & = 2^{m-m_0},                                                                              
				\end{align*}
				which implies that $ 2^{m-m_0}\cdot (2m_0+1)^d > (2m+1)^d$, therefore by the previous derivation, $2^m>(2m+1)^d$ as wanted.
				Next, assume $d\geq 30$ and set~${m_0=2d\log d}$. We will show $m_0$ satisfies items (i) and (ii). 
				We first prove that $m_0\geq \frac{d}{\ln 2}$:
				\begin{align*}
					m_0\geq \frac{d}{\ln 2} & \iff 2d\log d \geq \frac{d}{\ln 2}   \\
					                        & \iff  \log d   \geq \frac{1}{2\ln 2} \\ 
					                        & \iff d  \geq 2^{\frac{1}{2\ln 2}},   
				\end{align*}
				and indeed $d\geq 2^{\frac{1}{2\ln 2}}=\sqrt{e}$ since $d\geq 30$.
				Lastly, we prove that ${2^{m_0} \geq (2m_0+1)^d}$:
				\begin{align*}
					2^{m_0}=2^{2d\log d} & = d^{2d}                                                          \\ 
					                     & \geq (6d\log d)^d  \tag{since $d\geq 6\log d$ for all $d\geq 30$} \\
					                     & = (3m_0)^d                                                        \\
					                     & \geq(2m_0+1)^d.                                                   
				\end{align*}
			\end{proof}
					
			\subsection{Proof of Lemma~\ref{l:tech_fin_FCD_implies_fin_repdim}}

   \begin{lemma*}[Lemma~\ref{l:tech_fin_FCD_implies_fin_repdim}, Restatement]
			Let $\alpha\geq2$, and set $m=\lfloor20\alpha\ln \alpha\rfloor$ and $k=4m^\alpha$. Then,
			\[\bigl(1-q(m)\bigr)^{k}\leq \frac{1}{4},\]
   where $q(m)=\frac{1}{m^\alpha}-\left(\frac{3}{4}\right)^m .$
		\end{lemma*}
					
			\begin{proof}
				Observe that 
				\[\left(1-q(m)\right)^{k} \leq \exp(-k\cdot q(m)) \tag{$1+x < \exp(x)$ for all $x>0$}\]
				Hence it is enough to show that for $m=\lfloor 20\alpha\ln \alpha\rfloor$, $k=4m^\alpha$ 
				\begin{equation*} 
					k\cdot q(m)\geq\ln4.     
				\end{equation*}
				Indeed, it easy to verify that $\left(\frac{3}{4}\right)^m\leq \frac{1}{2m^\alpha}$ where $\alpha\geq \alpha$ and $m=\lfloor 20\alpha\ln \alpha\rfloor$. Therefore, 
				\begin{align*}
					k\cdot q(m)  =k\left(\frac{1}{m^\alpha}-\left(\frac{3}{4}\right)^m\right) &\geq k\cdot \frac{1}{2m^\alpha}                              \\
					            & = 2 \tag{setting $k=4m^\alpha$}                              \\
					            & \geq \ln 4.      
				\end{align*}
			\end{proof}

\end{document}